\documentclass{article}
\PassOptionsToPackage{numbers,sort&compress}{natbib}
\usepackage{xcolor}
\usepackage[preprint]{corl_2026}

\usepackage[utf8]{inputenc}
\usepackage[T1]{fontenc}
\usepackage{hyperref}
\usepackage{url}
\usepackage{booktabs}
\usepackage{amsfonts}
\usepackage{nicefrac}
\usepackage{microtype}
\usepackage{enumitem}
\usepackage{multicol}
\usepackage{graphicx}
\usepackage{makecell}
\usepackage{multirow}
\usepackage{floatrow}
\usepackage{wrapfig}
\usepackage{soul}
\usepackage{amsmath}
\usepackage{subcaption}
\usepackage{lipsum}
\usepackage{colortbl}
\usepackage{svg}
\usepackage{bbding}
\usepackage{dsfont}

\usepackage{amssymb}
\usepackage{amsthm}
\usepackage{bbm}
\usepackage{bm}
\usepackage{comment}
\usepackage{algorithm}
\usepackage{algpseudocode}

\newtheorem{theorem}{Theorem}

\newtheorem{lemma}[theorem]{Lemma}
\theoremstyle{definition}

\newtheorem{definition}{Definition}

\usepackage{array}
\newcolumntype{C}[1]{>{\centering\arraybackslash}p{#1}}

\title{MASkillBlender: Decentralized Whole-Body Coordination for Multi-Humanoid Loco-Manipulation via Skill Blending}
\author{
\textbf{Yifan Hu}\textsuperscript{1} \quad
\textbf{Luhang Hong}\textsuperscript{1} \quad
\textbf{Mingkang Long}\textsuperscript{1} \quad
\textbf{Danning Wang}\textsuperscript{1} \\[0.3em]
\textbf{Chengfeng Jia}\textsuperscript{1} \quad
\textbf{Rong Su}\textsuperscript{1}\thanks{Correspondence to: Rong Su (RSu@ntu.edu.sg)} \quad
\textbf{Junjie Fu}\textsuperscript{2,3} \quad
\textbf{Guanghui Wen}\textsuperscript{2} \\[0.5em]
{\small\textsuperscript{1}Nanyang Technological University \quad
\textsuperscript{2}Southeast University} \\
{\small\textsuperscript{3}Purple Mountain Laboratories} \\[0.3em]
{\small Project page: \url{https://maskillblender.github.io/}}
}
\date{}
\hypersetup{
  pdftitle={MASkillBlender: Decentralized Whole-Body Coordination for Multi-Humanoid Loco-Manipulation via Skill Blending},
  pdfauthor={Yifan Hu, Luhang Hong, Mingkang Long, Danning Wang, Chengfeng Jia, Rong Su, Junjie Fu, Guanghui Wen},
  pdfsubject={}
}
\begin{document}
\normalsize

\maketitle

\vspace{-20pt}
\begin{figure*}[h]
\centering
\includegraphics[width=0.9\linewidth]{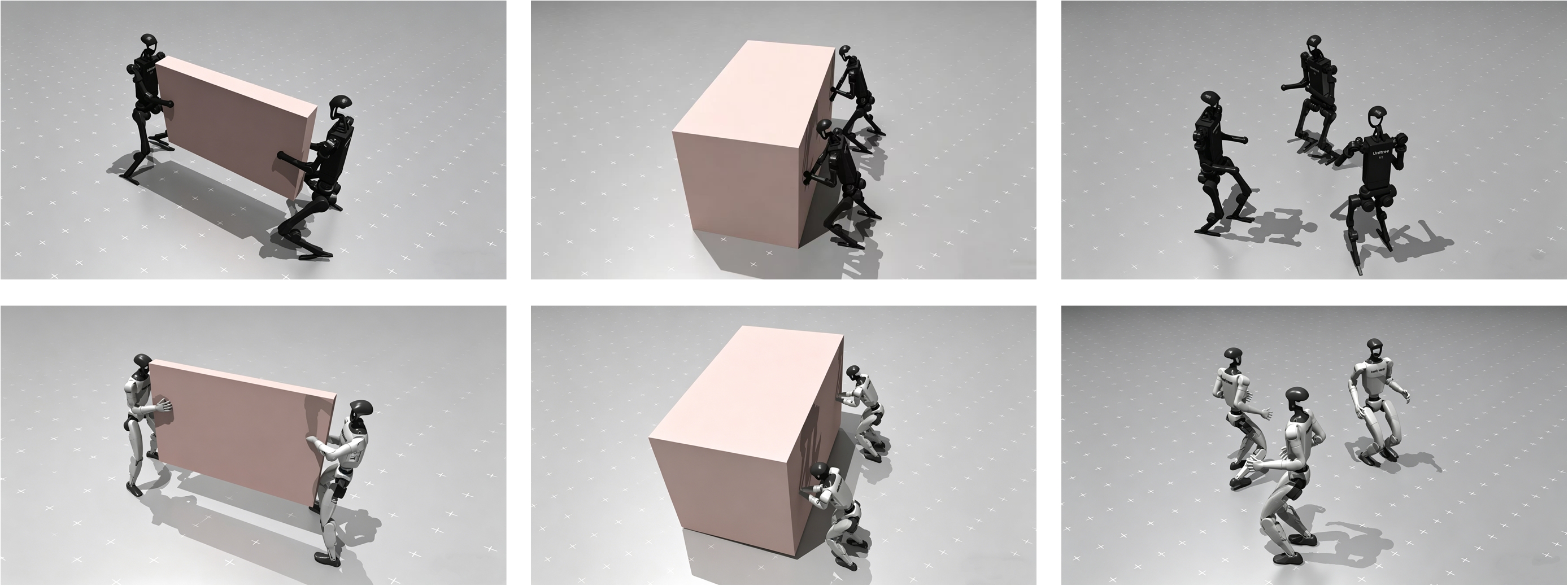}
\caption{MASkillBlender enables decentralized multi-humanoid whole-body coordination across diverse tasks and humanoid embodiments without requiring task-specific motion references.}
\label{fig:versatile_task}
\end{figure*}

\begin{abstract}
Coordinated multi-humanoid loco-manipulation is promising yet challenging due to high-dimensional whole-body control, decentralized decision making, and scalability. While recent reinforcement learning methods have improved single-humanoid whole-body control, extending them to the multi-humanoid setting remains nontrivial and often requires substantial reward engineering or task-specific design. We propose MASkillBlender, a general multi-agent reinforcement learning framework to achieve decentralized multi-humanoid whole-body coordination. By learning a shared decentralized high-level policy over reusable pre-trained single-humanoid skills, MASkillBlender enables coordinated behaviors using only task-level rewards, without requiring task-specific motion references. To improve learning efficiency, we further introduce a permutation-based data augmentation strategy for homogeneous multi-humanoid systems, and theoretically show that the permuted samples preserve the policy-gradient direction of the original samples under the homogeneous Markov game formulation. We evaluate MASkillBlender on multiple multi-humanoid coordination tasks across two humanoid embodiments. Simulation results demonstrate that the proposed framework consistently achieves strong task performance and enables coordinated behaviors across different tasks and humanoid embodiments.
\end{abstract}


\section{Introduction}  \label{intro}
Humanoid robots have substantial potential to replace or assist humans in completing daily tasks across diverse environments due to their human-like morphology~\citep{gu2025humanoid}. With advancements in reinforcement learning (RL) and robot hardware, the loco-manipulation capabilities of humanoids have been significantly enhanced~\citep{he2025asap}. Nevertheless, many real-world tasks require multiple humanoids to coordinate their actions~\citep{zhang2023simulation}. Compared with single-humanoid control, multi-humanoid settings introduce additional challenges due to complex inter-agent dependencies. Developing general learning methods for coordinated multi-humanoid control therefore represents an important step toward scalable and capable humanoid teams.

Recent studies have demonstrated increasingly sophisticated multi-humanoid whole-body interactions, showing the potential of learning-based approaches for coordinated humanoid control~\citep{feng2026multiagent}. However, existing methods typically rely on human motion references to guide policy learning through motion imitation or tracking objectives. While such motion priors provide effective guidance for acquiring natural whole-body behaviors, they also introduce additional data and supervision requirements, and the resulting motion-based objectives may compete with task rewards during optimization, potentially constraining the policy from discovering coordination strategies that are most effective for the task~\citep{mo2026learning}. Beyond the reliance on motion supervision, incorporating observations of other humanoids into each agent's policy has been widely adopted to facilitate coordination. While such inter-agent information can improve learning and coordination performance, it may hinder scalability to larger teams and complicate robust real-world deployment.

For single-humanoid loco-manipulation, skill-based learning has emerged as an effective paradigm that does not rely on task-specific motion references~\citep{sferrazza2024humanoidbench, kuang2025skillblender}. By blending pre-trained primitive skills into the control policy, humanoids can reuse existing motor capabilities and reduce the burden of low-level control learning during downstream task training. However, extending this paradigm to multi-humanoid control is not straightforward, since skills pre-trained for individual humanoids do not encode the inter-agent coordination required by downstream tasks. Learning how to compose these skills across agents under coupled task dynamics therefore remains challenging.

In this work, we propose MASkillBlender, a general skill-based multi-agent RL (MARL) framework for decentralized coordination in homogeneous multi-humanoid systems. Under centralized training with decentralized execution (CTDE), a shared decentralized high-level policy learns to coordinate the composition of reusable pre-trained primitive skills across humanoids using local observations. The primitive skills remain fixed during downstream training, allowing the policy to focus on high-level skill composition and inter-agent coordination using only task-level rewards. To improve learning efficiency, we introduce a permutation-based data augmentation strategy that exploits the permutation symmetry of homogeneous multi-humanoid systems. To evaluate the proposed framework, we build three representative multi-humanoid coordination tasks spanning cooperative object manipulation and coordinated locomotion across two distinct humanoid embodiments, using only task-level rewards and limited local observations. To summarize, the main contributions of this work are as follows:

\begin{itemize}[noitemsep, leftmargin=*]
    \item We propose MASkillBlender, a general skill-based MARL framework that learns decentralized multi-humanoid coordination over reusable pre-trained single-humanoid skills using only task-level rewards, without requiring task-specific motion references.
    \item We incorporate a permutation-based data augmentation strategy into MASkillBlender and theoretically show that the augmented samples preserve the policy-gradient direction of the original samples.
    \item We evaluate the proposed MARL framework on diverse multi-humanoid coordination tasks across different embodiments, and extensive simulation results show that the learned policies produce versatile coordinated behaviors and effectively handle all tasks in a fully decentralized manner.
\end{itemize}

\section{Related Works}  \label{related_work}

\paragraph{Single-Humanoid Whole-Body Control with RL.}
RL has become a key paradigm for humanoid whole-body control, facilitating complex locomotion and manipulation~\citep{radosavovic2024jitendra,gu2024advancing,chen2024learning,chen2026physics,lin2025sim}. Existing RL-based methods can be broadly categorized according to whether the control policy relies on task-specific motion references. Reference-conditioned methods learn versatile whole-body behaviors by tracking motions obtained from MoCap datasets~\citep{cheng2024expressive,ji2024exbody2,yang2025omniretarget} or human teleoperation~\citep{fu2024humanplus,he2024h2o,he2024omnih2o,sun2025ulc,ze2025twist,ben2025homie}. While effective for single-humanoid loco-manipulation, extending this paradigm to multi-humanoid coordination requires coordinated motion references that are substantially more difficult to acquire. In contrast, task-driven methods learn loco-manipulation behaviors without conditioning on task-specific reference motions~\citep{sferrazza2024humanoidbench,kuang2025skillblender,dao2024sim,zhang2024wococo}. To address the difficulty of learning in high-dimensional state-action spaces, these methods typically employ task decomposition~\citep{dao2024sim,zhang2024wococo} or skill-based hierarchical control architectures~\citep{sferrazza2024humanoidbench,kuang2025skillblender}. For the latter, pre-trained low-level skills are reused and composed by a high-level policy to reduce the burden of learning whole-body control from scratch, which empirically improves learning efficiency and enables diverse loco-manipulation behaviors across downstream tasks.

\paragraph{Multi-Humanoid Whole-Body Coordination with RL.}
Learning coordinated whole-body behaviors among multiple humanoids has recently received growing attention. SMPLOlympics~\citep{luo2024smplolympics} explores competitive interactions between physically simulated humanoids in diverse sports environments. For cooperative coordination, reference-based methods retarget multi-person motions to humanoids and learn interaction policies through motion tracking and interaction-specific objectives~\citep{liu2025takes,chen2026rhythm}, enabling complex whole-body interactions such as contact-rich social behaviors and synchronized motions. To alleviate the scarcity of multi-person motion data, CooHOI~\citep{gao2024coohoi} and TeamHOI~\citep{lionar2026teamhoi} instead leverage single-human motion priors to facilitate cooperative object manipulation. CooHOI transfers single-agent object-interaction skills to cooperative settings, while TeamHOI employs masked motion priors and a teammate-aware policy to support cooperative carrying with varying team sizes. Note that these methods are typically task-specific and require high-quality motion retargeting to obtain suitable humanoid references. CDA-PPO~\citep{baltes2025cooperative} follows a conventional MARL formulation and learns cooperative object transportation from scratch without motion supervision. However, it is specifically designed for a two-humanoid setting and relies on extensively engineered locomotion and cooperation rewards as well as explicit teammate observations. As a result, it remains a huge challenge to develop general reference-free RL methods for diverse multi-humanoid whole-body coordination with lightweight reward engineering.

\section{Preliminaries}  \label{preliminaries}

\subsection{Homogeneous Markov Games}
Markov Games (MGs) provide a general framework for formulating multi-agent policy learning problems~\citep{zhang2021multi}. An MG is defined as $\langle N, \mathcal{S}, \{\mathcal{A}_i\}_{i=1}^N, \{R_i\}_{i=1}^N, p, \gamma \rangle$, where $N$ is the number of agents, $\mathcal{S}$ is the state space, $\mathcal{A}_i$ is the action space of agent $i$ with $\mathcal{A} = \prod_{i=1}^N \mathcal{A}_i$ being the joint action space, and $R_i: \mathcal{S} \times \mathcal{A} \to \mathbb{R}$ is the corresponding reward function. The transition function $p: \mathcal{S} \times \mathcal{A} \times \mathcal{S} \to [0,+ \infty)$ specifies the environment dynamics, and $\gamma \in [0, 1)$ is the discount factor. For homogeneous agents, a subclass of MGs, termed H-MGs~\citep{chen2022iclr,hu2024distributed}, can be defined as follows:

\begin{definition}  \label{definition:Homogeneous Markov game}
   An MG $\langle N, \mathcal{S}, \{\mathcal{A}_i\}_{i=1}^N, \{R_i\}_{i=1}^N, p, \gamma \rangle$ is an H-MG if it satisfies the following conditions:
   \begin{itemize} [leftmargin=10pt, labelsep=.5em, align=left]
       \item \label{item:homogeneous state-action spaces}
       The state space is the product space of $N$ homogeneous local state spaces, i.e., $s = (s_1, \dots, s_N) \in \mathcal{S} = \mathcal{S}_1 \times \cdots \times \mathcal{S}_N$ with $\mathcal{S}_i = \mathcal{S}_j$, $\forall i, j \in \mathcal{I}$, where $s_i \in \mathcal{S}_i$, $\forall i \in \mathcal{I}$. The local action spaces of all agents are identical, i.e., $\mathcal{A}_i = \mathcal{A}_j$, $\forall i, j \in \mathcal{I}$\footnote{All notations are summarized in Appendix \ref{appendix_notation}.}.
       \item \label{item:permutation invariant and permutation preserving}
       The joint reward function $R(s_t, a_t) = (R_1(s_t, a_t), \dots, R_N(s_t, a_t))$ is permutation preserving, and the transition function is permutation invariant, i.e., for any $M \in \mathcal{M}$, $s_t = (s_{1,t}, \dots, s_{N,t}) \in \mathcal{S}$, and $a_t = (a_{1,t}, \dots, a_{N,t}) \in \mathcal{A}$, it holds
       \begin{align*}
            R(M s_t, M a_t) = M R(s_t, a_t), \quad p(s_{t+1} | s_t, a_t) =  p(M s_{t+1}|M s_t, M a_t).
       \end{align*}
       \item \label{item:observation permutation preserving}
       Each agent $i \in \mathcal{I}$ obtains its local observation through a surjective function $o_i: \mathcal{S} \to \mathcal{O}$, where the observation space $\mathcal{O}$ is shared by all agents. Furthermore, the joint observation $o(s) = (o_1(s), \dots, o_N(s))$ is permutation preserving, i.e., for any $s \in \mathcal{S}$ and $M \in \mathcal{M}$, it holds
       \begin{align*}
           o(M s) = M o(s).
       \end{align*}
   \end{itemize}
\end{definition}

\subsection{Problem Formulation}
We formulate the multi-humanoid whole-body control policy learning problem as a MARL problem under the H-MG framework. Notice that H-MGs are capable of modeling diverse multi-humanoid coordination tasks, where one example is provided in Appendix~\ref{appendix_example} to give more intuition about this formulation. Following prior work on single-humanoid RL~\citep{zhang2024wococo,kuang2025skillblender}, the local observation of each humanoid consists of its proprioceptive state and a local task goal. The local action corresponds to target joint positions of the humanoid's whole body, which are mapped to motor torques through PD control. Moreover, all humanoids share a team reward function, denoted by $R_{\mathrm{team}}$, such that $R_{\mathrm{team}} = R_i$, $\forall i \in \mathcal{I}$. Let $\pi^{\mathrm{tot}}: \mathcal{S} \times \mathcal{A} \to [0, \infty)$ be the joint policy, based on which the value function is defined as $V_{\pi^{\mathrm{tot}}}(s) = \mathbb{E}\left[\sum_{t=0}^\infty \gamma^t R_{\mathrm{team}}(s_t, a_t)|s_0 = s\right]$. The target is to find an optimal policy $\pi^{\mathrm{tot}}_*$ that maximizes the objective function $J(\pi^{\mathrm{tot}}) = \mathbb{E}_{s_0 \sim \rho}\left[V_{\pi^{\mathrm{tot}}}(s_0)\right]$, where $\rho$ denotes the initial state distribution over $\mathcal{S}$.

\begin{figure*}[!t]

\centering
\includegraphics[width=1.0\linewidth]{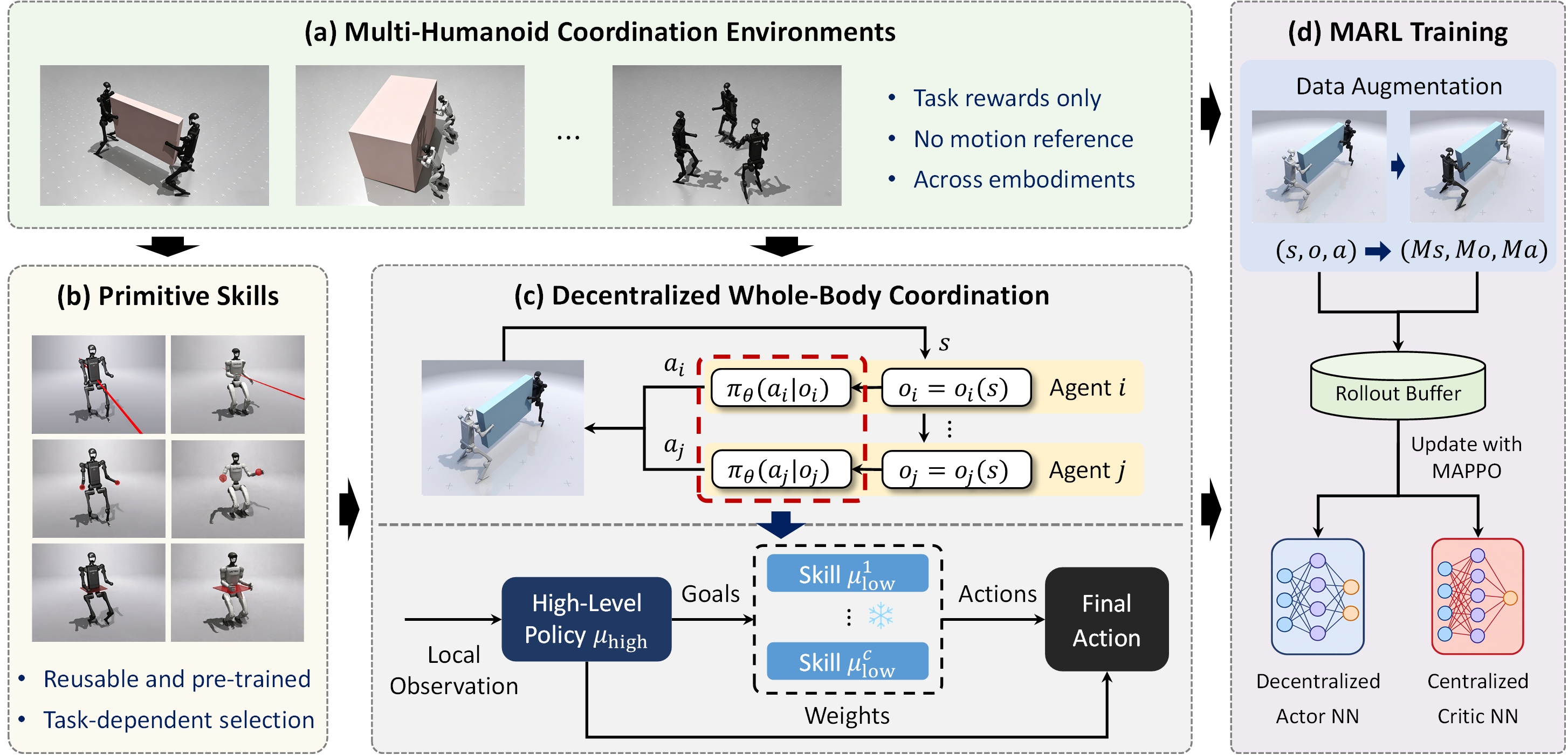}
\caption{Overview of MASkillBlender. (a) MASkillBlender enables diverse multi-humanoid coordination across different humanoid embodiments without task-specific motion references or extensive reward engineering. (b)-(c) For each task, relevant pre-trained primitive skills are selected and blended by a shared decentralized high-level policy. (d) The high-level policy is trained with MAPPO using permutation-based data augmentation.}
\label{fig:maskillblender_framework}
\end{figure*}

\section{Methodology} \label{sec:methodology}
In practice, it is challenging to learn coordinated humanoid whole-body control policies with vanilla MARL. Without task-specific motion references, learning from scratch requires the policy to simultaneously acquire low-level motor control and inter-agent coordination. Moreover, the data-collection cost increases substantially as more humanoids are included into the environment, making iterative reward engineering particularly costly. To deal with these challenges, we propose MASkillBlender (see Fig.~\ref{fig:maskillblender_framework}), which combines MARL with reusable pre-trained skills for decentralized multi-humanoid coordination and permutation-based data augmentation for improved data efficiency. We first introduce the decentralized skill-blended control policy in Sec.~\ref{subsec:policy}, then the permutation-based data augmentation in Sec.~\ref{subsec:data_aug}, and finally the overall MARL framework in Sec.~\ref{subsec:final_rl_framework}.

\subsection{Decentralized Skill-Blended Control Policy} \label{subsec:policy}
As shown in Fig.~\ref{fig:maskillblender_framework}(c), we adopt decentralized control policies for scalable multi-humanoid coordination. We employ policy sharing~\citep{gupta2017cooperative} due to its empirical effectiveness in MARL~\citep{hu2023graph} and theoretical support for optimality in homogeneous settings~\citep{chen2022iclr}. Accordingly, the joint policy factorizes as $\pi^{\mathrm{tot}}(a_t|s_t)=\prod_{i=1}^{N}\pi(a_{i,t}|o_i(s_t))$, where $\pi$ denotes an observation-based local policy shared by all humanoids. Let $s_{i,t}^{\mathrm{prop}}$ and $g_{i,t}$ denote the proprioceptive state and the local task goal of humanoid $i$, respectively. Then, the local observation is given by $o_i(s_t)=(s_{i,t}^{\mathrm{prop}},g_{i,t})$. Note that the task goal $g_{i,t}$ may vary across humanoids but has the same dimensionality for all agents.

Multi-humanoid loco-manipulation can be viewed as coordinating the individual loco-manipulation capabilities of multiple humanoids toward a shared task. Recent work on single-humanoid whole-body control has shown that complex loco-manipulation behaviors can be composed from reusable motor skills rather than learned entirely from scratch~\citep{kuang2025skillblender}. Following this intuition, we build the decentralized local policy $\pi$ upon a set of pre-trained primitive skills. For each humanoid embodiment, we construct a reusable skill library consisting of task-agnostic primitive skills (see Fig.~\ref{fig:maskillblender_framework}(b)), including \texttt{Walking}, \texttt{Reaching}, and \texttt{Squatting}, which provide basic motor capabilities for synthesizing diverse whole-body behaviors. These skills are pre-trained using dense rewards to capture low-level control objectives such as whole-body stability and energy efficiency, thereby reducing the requirement to explicitly encode these objectives in the reward design for multi-humanoid coordination. Brief descriptions of these primitive skills are provided in Appendix~\ref{appendix_description_skills}. Note that the skill library can be extended with additional primitive skills to support more complex coordination behaviors.

Following~\citet{kuang2025skillblender}, we adopt a hierarchical policy architecture to blend primitive skills. For each coordination task, we first select $C$ relevant primitive skills from the skill library according to the task requirements. For instance, in a box-pushing scenario, the coordinated behavior mainly relies on the \texttt{Walking} and \texttt{Reaching} skills, whereas \texttt{Squatting} contributes little. Accordingly, we select only \texttt{Walking} and \texttt{Reaching} as the low-level skills for our hierarchical policy. Let $\mu_{\mathrm{low}}^c$ denote the $c$-th primitive skill. At each time instant $t$, given the proprioceptive state $s_{i,t}^{\mathrm{prop}}$ and the skill-dependent goal $g^c_{i,t} \in \mathbb{R}^{d_c}$ as inputs, $\mu_{\mathrm{low}}^c$ outputs a deterministic local action for humanoid $i$, given by $u^c_{i,t} = \mu_{\mathrm{low}}^c(s_{i,t}^{\mathrm{prop}}, g^c_{i,t}) \in \mathbb{R}^{d}$, where $u^c_{i,t}$ denotes the target joint positions and $d$ denotes the number of DoFs of the humanoid. Note that the dimension of $g^c_{i,t}$ varies across skills.

After the primitive skills are selected, a decentralized high-level policy is built to blend these skills, which is illustrated in Fig.~\ref{fig:maskillblender_framework}(c). Specifically, at each time instant $t$, humanoid $i$'s high-level policy $\mu_{\mathrm{high}}$ takes the local observation $o_i(s_t)$ as its input, and outputs a raw joint goal vector $\tilde g_{i,t}$ and a raw joint weight vector $\tilde w_{i,t}$, i.e., $[\tilde g_{i,t}^T, \tilde w_{i,t}^T]^T = \mu_{\mathrm{high}}(o_i(s_{t}))$, where $\tilde g_{i,t} = [(\tilde g^1_{i,t})^T, \dots, (\tilde g^C_{i,t})^T]^T$ with $\tilde g^c_{i,t} \in \mathbb{R}^{d_c}$ and $\tilde w_{i,t} = [(\tilde w^1_{i,t})^T, \dots, (\tilde w^C_{i,t})^T]^T$ with $\tilde w^c_{i,t} = [\tilde w^c_{i,t,1}, \dots, \tilde w^c_{i,t,d}]^T \in \mathbb{R}^{d}$ for any $c$. To provide reasonable goal information, a skill-specific clipping function is applied to $\tilde g_{i,t}$ to produce the final goal vector $g^c_{i,t}$, which is then used by each primitive skill $\mu_{\mathrm{low}}^c$ to generate its action $u^c_{i,t}$. To avoid exaggerated values when blending these actions, each raw weight vector $\tilde w^c_{i,t}$ is processed through a pointwise Softmax function. Let $w_{i,t}^c = [w^c_{i,t,1}, \dots, w^c_{i,t,d}]^T$ denote the processed weight vector of the $c$-th primitive skill, where each $w^c_{i,t,k}$ satisfies $w^c_{i,t,k} = \mathrm{softmax}_c([\tilde w^1_{i,t,k}, \dots, \tilde w^C_{i,t,k}]^T)$. Based on the weight vectors $\{w_{i,t}^c\}_{c=1}^C$, all actions $\{u^c_{i,t}\}_{c=1}^C$ are blended to generate the final local action $u_{i,t} = [u_{i,t,1}, \dots, u_{i,t,d}]^T$, where each $u_{i,t,k}$ satisfies
\begin{align}
    u_{i,t,k} = \sum_{c=1}^C w_{i,t,k}^c \cdot u^c_{i,t,k}, \quad \quad \forall k = 1, \dots, d.
\end{align}
Note that both the primitive skills and the high-level policy require only local information. Therefore, the resulting hierarchical whole-body control policy is decentralized.

To facilitate RL-based policy training, we parameterize the high-level policy as $\mu_{\mathrm{high},\xi}$ and model the resulting decentralized policy as a $d$-dimensional diagonal Gaussian distribution: $\pi_{\theta}(\cdot | o_i(s_t)) = \mathcal{N}(u_{i,t}, \mathrm{diag}(\sigma^2))$, where $\theta = [\xi^T, \sigma^T]^T$ denotes the learnable parameters and $\sigma \in \mathbb{R}^d$ is the standard deviation vector. Denote the corresponding parameterized joint policy by $\pi^{\mathrm{tot}}_\theta$. Note that the primitive skills remain fixed during training, and their parameters are therefore excluded from $\theta$. This shifts the learning problem from acquiring low-level whole-body motor skills from scratch to learning high-level skill composition and inter-agent coordination over a fixed set of reusable skills, thereby structuring policy learning and making the MARL problem more tractable.

\subsection{Permutation-Based Data Augmentation} \label{subsec:data_aug}
RL-based humanoid control typically relies on large-scale parallel simulation~\citep{makoviychuk2021isaac}. Nevertheless, the simulation cost increases substantially in the multi-humanoid setting due to the larger number of agents and more complex inter-agent interactions. To improve data efficiency, we exploit the permutation properties of H-MGs~\citep{chen2022iclr} and incorporate a permutation-based data augmentation strategy into policy learning.

Denote $o_t = o(s_t)$. For any permutation $M \in \mathcal{M}$, we can apply a permutation transformation to the collected data sample $(s_t, o_t, a_t)$ to obtain the permuted counterpart, i.e.,
\begin{align}
    (s_t, o_t, a_t) \xrightarrow{\;\;M\;\;} (M s_t, M o_t, M a_t).
\end{align}
For notational convenience, let $V_{\theta}(s) = V_{\pi^{\mathrm{tot}}_\theta}(s)$, based on which the action-value function is defined as $Q_{\theta}(s,a) = R_{\mathrm{team}}(s, a) + \gamma \mathbb{E}_{s' \sim p(\cdot | s, a)}[V_{\theta}(s')]$~\citep{sutton2018book}. Under the shared decentralized policy, we consider a joint-policy surrogate objective $L_{\theta_k}(\theta)$ following the TRPO/PPO formulation~\citep{schulman2015trust,schulman2017proximal}. Its gradient can be expressed in the following importance-sampling form:
\begin{align} \label{importance_sampling_based_policy_gradient}
    \nabla_{\theta} L_{\theta_k}(\theta) = \mathbb{E}_{s_t \sim \rho_{\pi^{\mathrm{tot}}_{\theta_k}}, a_t \sim {\pi^{\mathrm{tot}}_{\theta_k}}} \left[\frac{\pi^{\mathrm{tot}}_\theta(a_t|s_t)}{\pi^{\mathrm{tot}}_{\theta_{k}}(a_t|s_t)} A_{\theta_k}(s_t,a_t)\, \nabla_{\theta}\log \pi^{\mathrm{tot}}_{\theta}(a_t | s_t) \right],
\end{align}
where $\pi^{\mathrm{tot}}_{\theta_k}$ is a rollout policy, and $\rho_{\pi^{\mathrm{tot}}_{\theta}}(s) = \sum_{t=0}^\infty \gamma^t \mathbb{P}_{\theta}(s_t = s)$ is the (unnormalized) visitation measure of $s$, where $\mathbb{P}_{\theta}(s_t = s)$ denotes the marginal state distribution w.r.t. the policy $\pi^{\mathrm{tot}}_{\theta}$, and $A_{\theta}(s_t,a_t) = Q_{\theta}(s_t,a_t) - V_{\theta}(s_t)$ is the advantage function. Denote the policy gradient induced by the data sample $(s_t, o_t, a_t)$ by
\begin{align}
    f(s_t,a_t,o_t;\theta,\theta_k) = \frac{\pi^{\mathrm{tot}}_\theta(a_t|s_t)}{\pi^{\mathrm{tot}}_{\theta_{k}}(a_t|s_t)} A_{\theta_k}(s_t,a_t)\, \nabla_{\theta}\log \pi^{\mathrm{tot}}_{\theta}(a_t | s_t).
\end{align}

\begin{lemma} \label{lemma_value_function}
    For an H-MG in which all agents share a team reward $R_{\mathrm{team}}$ and an observation-based local policy $\pi_{\theta}$, both the value function $V_{\theta}$ and the action-value function $Q_{\theta}$ are permutation invariant, i.e., $V_{\theta}(s) = V_{\theta}(Ms)$ and $Q_{\theta}(s,a) = Q_{\theta}(Ms, Ma)$, $\forall s \in \mathcal{S}$, $\forall a \in \mathcal{A}$, $\forall M \in \mathcal{M}$.
\end{lemma}

\begin{theorem} \label{theorem_data_augmentation}
    Under the setting of H-MGs in Lemma~\ref{lemma_value_function}, for any rollout joint policy $\pi^{\mathrm{tot}}_{\theta_k}$, any collected sample $(s_t, o_t, a_t)$ and its permuted counterpart $(M s_t, M o_t, M a_t)$ yield the same policy gradient, i.e.,
    \begin{align}  \label{theorem_policy_gradient}
        f(s_t,a_t,o_t;\theta,\theta_k) = f(M s_t, M a_t, M o_t;\theta,\theta_k), \quad \forall M \in \mathcal{M}.
    \end{align}
\end{theorem}

The proofs of Lemma~\ref{lemma_value_function} and Theorem~\ref{theorem_data_augmentation} can be found in Appendix~\ref{appendix_proof}. Theorem~\ref{theorem_data_augmentation} clearly indicates that augmenting a batch of collected samples with their permuted counterparts leaves the policy-gradient direction unchanged, thereby providing theoretical justification for incorporating permutation-based data augmentation into our MARL framework.

Note that prior work has provided a preliminary investigation of permutation-based experience augmentation in off-policy MARL~\citep{ye2022improving}, where the theoretical analysis only focuses on the validity of permuted transitions and Q-function estimation, while empirical evaluation is limited to low-dimensional multi-agent benchmarks. In contrast, we extend permutation-based augmentation to on-policy MARL for high-dimensional multi-humanoid coordination and establish its theoretical validity at the policy-gradient level under the H-MG formulation. In addition, our method is orthogonal to other data augmentation techniques, such as symmetry-based augmentation~\citep{lin2020invariant,mittal2024symmetry}. We leave the integration of these approaches for future work.

\subsection{MASkillBlender}  \label{subsec:final_rl_framework}
We finally integrate the decentralized skill-blended control policy and permutation-based data augmentation into the overall MASkillBlender framework. Specifically, we adopt MAPPO~\citep{yu2022surprising} to train the shared decentralized high-level policy. During the training process, a centralized critic $V_{\phi}(s_t)$, parameterized by $\phi$, takes the global state as input to approximate the value function $V_{\theta}(s_t)$ associated with the current joint policy $\pi^{\mathrm{tot}}_\theta$. The centralized critic is used only for policy training, while each humanoid executes the shared policy independently using its local observation.

During policy optimization, for each minibatch $\mathcal{B}$ sampled from the rollout buffer, we select one or more permutations from $\mathcal{M}$. For each permutation $M$, we augment $\mathcal{B}$ with the permuted counterpart $(Ms_t, Mo_t, Ma_t)$ of each sample $(s_t, o_t, a_t)$. Since $V_{\theta}(s)$ is permutation invariant by Lemma~\ref{lemma_value_function}, the centralized critic $V_{\phi}$ is trained using the augmented minibatch, such that $V_{\phi}(s_t)$ and $V_{\phi}(Ms_t)$ are optimized to approximate the same return. The overall training procedure of MASkillBlender is summarized in Appendix~\ref{appendix_maskillblender_training}.

\section{Experiments}  \label{sec:experiments}
We conduct experiments to evaluate the effectiveness, robustness, and generality of MASkillBlender across different multi-humanoid coordination tasks. The experimental setup is introduced in Sec.~\ref{subsec:experimental_setup}. Main results are presented in Sec.~\ref{subsec:main_results}, followed by ablation studies in Sec.~\ref{subsec:ablation}, policy robustness evaluations in Sec.~\ref{subsec:policy_robustness}, and cross-simulator transfer in Sec.~\ref{subsec:sim2sim}. Additional experiments are provided in Sec.~\ref{subsec:additional_experiments}.

\begin{figure*}[t]
    \centering
    \includegraphics[width=\linewidth]{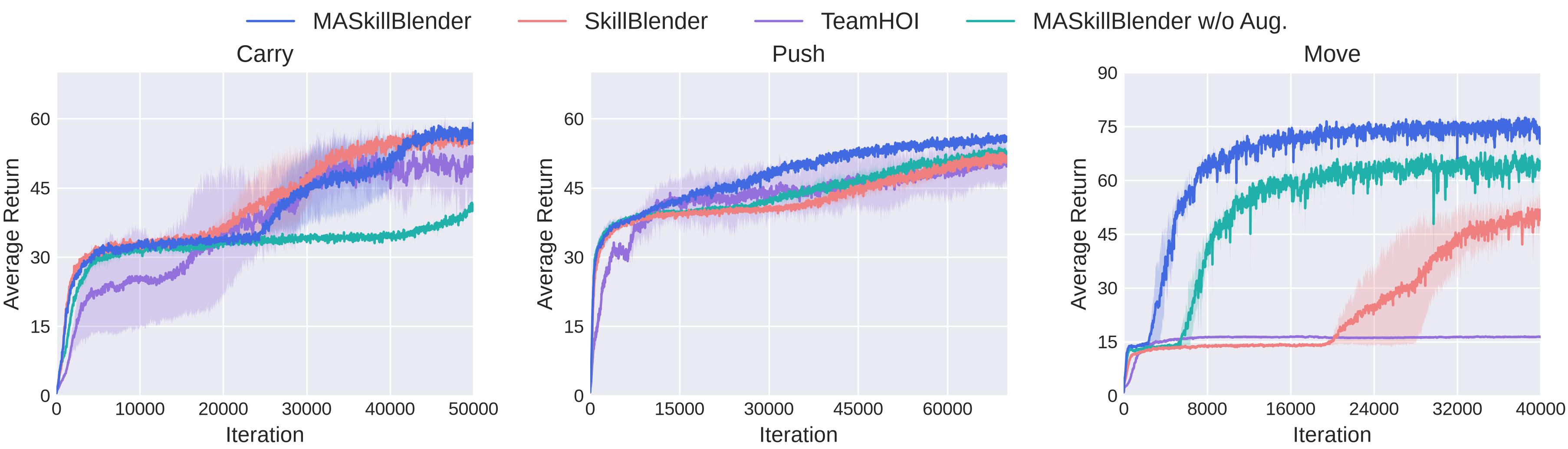}
    \caption{Training curves of all methods on three Unitree H1 tasks.}
    \label{fig:training_curves_h1}
\end{figure*}

\subsection{Experimental Setup}  \label{subsec:experimental_setup}

\paragraph{Tasks.}
Inspired by classic multi-robot coordination problems, we consider three multi-humanoid coordination tasks, namely \texttt{Carry}~\citep{gasbarrone2024decentralized}, \texttt{Push}~\citep{tang2024collaborative}, and \texttt{Move}~\citep{hu2022decentralized}. All tasks are implemented in SkillBench~\citep{kuang2025skillblender}, which is implemented in NVIDIA Isaac Gym~\citep{makoviychuk2021isaac}. In \texttt{Carry}, two humanoids collaboratively transport a large box to a target position in the air. In \texttt{Push}, two humanoids collaboratively push a heavy box toward a target position on the ground plane. Both tasks involve contact-rich interactions, where the humanoids are coupled through the motion of the shared object. In \texttt{Move}, three humanoids move toward their assigned target positions while avoiding inter-humanoid collisions. Detailed task specifications, including observations, actions and reward functions, are provided in Appendices~\ref{appendix_spaces} and~\ref{appendix_task_specifications}. For all three tasks, the local policy observation is deliberately kept compact, consisting only of proprioceptive information and the task-relevant information necessary for local decision making. In addition, each task is specified with only a few reward terms, all of which are directly associated with task completion. This reduces the need for extensive reward engineering. As shown in Fig.~\ref{fig:versatile_task}, each task supports two humanoid models, namely the 19-DoF Unitree H1 and the 21-DoF Unitree G1. Detailed specifications of these two humanoid models are provided in Appendix~\ref{appendix_model_details}.

\paragraph{Baselines and Implementation Details.}
We compare our method with SkillBlender~\citep{kuang2025skillblender} and TeamHOI~\citep{lionar2026teamhoi}. SkillBlender is originally designed for single-humanoid whole-body loco-manipulation. We adapt it to the multi-humanoid setting by using a centralized high-level policy that takes the joint observations of all humanoids as input and jointly predicts their skill goals and blending weights. Accordingly, the adapted SkillBlender can be regarded as a centralized counterpart to our method. Note that SkillBlender employs the same pre-trained primitive skills as our method in all tasks. TeamHOI is a reference-based method for cooperative multi-humanoid object transport. It trains a decentralized Transformer policy with PPO and employs a masked Adversarial Motion Prior (AMP) constructed from single-human reference motions to regularize whole-body behaviors during coordination. For fair comparison, we keep the task environments and task-level information consistent across methods whenever possible, while retaining the method-specific inputs required by each baseline. We provide the implementation details of our method and two baselines in Appendix~\ref{appendix_implementation_details}.

\paragraph{Evaluation Metrics.}
To comprehensively evaluate the learned policies, our metrics include the average episodic return, denoted by Return~($\uparrow$), and a set of task-specific evaluation metrics. For all methods, Return is computed using only the task rewards defined in our benchmark environments, excluding auxiliary rewards used during training. For both \texttt{Carry} and \texttt{Push}, the task metric is the distance between the final box position and the target position, denoted by Error$_{\mathrm{box}}$~($\downarrow$). For \texttt{Move}, we report the maximum distance between the humanoids' final root positions and their respective targets, as well as the cumulative heading error over the episode, denoted by Error$_{\mathrm{pos}}$~($\downarrow$) and Error$_{\mathrm{head}}$~($\downarrow$), respectively. Based on these metrics, we further define task-specific success criteria and report the corresponding success rate, denoted by Success~($\uparrow$). For \texttt{Carry}, an episode is considered successful if Error$_{\mathrm{box}} \leq 0.1$. For \texttt{Push}, an episode is considered successful if Error$_{\mathrm{box}} \leq 0.2$. For \texttt{Move}, an episode is considered successful if both Error$_{\mathrm{pos}} \leq 0.03$ and no inter-humanoid collision occurs during the episode.

\begin{table*}[t]
\centering
\scriptsize
\setlength{\tabcolsep}{3pt}
\renewcommand{\arraystretch}{0.90}

\caption{Quantitative comparison of MASkillBlender with SkillBlender and TeamHOI on the Unitree H1 tasks.}
\label{tab:main_results}

\begin{tabular}{C{1.3cm}C{1.8cm}C{2.4cm}C{2.4cm}C{2.4cm}}
\toprule
Task & Metric & MASkillBlender & SkillBlender & TeamHOI \\
\midrule

\multirow{3}{*}{\texttt{Carry}}
& Return~($\uparrow$)
& \textbf{67.928 $\pm$ 1.563}
& 65.855 $\pm$ 8.343
& 67.870 $\pm$ 9.935 \\

& Error$_{\mathrm{box}}$~($\downarrow$)
& \textbf{0.016 $\pm$ 0.008}
& 0.025 $\pm$ 0.050
& 0.029 $\pm$ 0.102 \\

& Success~($\uparrow$)
& \textbf{100.0\%}
& 98.0\%
& 96.0\% \\
\cmidrule(lr){1-5}

\multirow{3}{*}{\texttt{Push}}
& Return~($\uparrow$)
& \textbf{61.137 $\pm$ 5.400}
& 57.550 $\pm$ 5.011
& 57.991 $\pm$ 5.024 \\

& Error$_{\mathrm{box}}$~($\downarrow$)
& \textbf{0.116 $\pm$ 0.063}
& 0.138 $\pm$ 0.069
& 0.136 $\pm$ 0.075 \\

& Success~($\uparrow$)
& \textbf{86.0\%}
& 84.0\%
& 78.0\% \\
\cmidrule(lr){1-5}

\multirow{4}{*}{\texttt{Move}}
& Return~($\uparrow$)
& \textbf{90.622 $\pm$ 1.022}
& 71.190 $\pm$ 3.211
& 16.723 $\pm$ 0.256 \\

& Error$_{\mathrm{pos}}$~($\downarrow$)
& \textbf{0.010 $\pm$ 0.003}
& 0.043 $\pm$ 0.012
& 2.755 $\pm$ 0.329 \\

& Error$_{\mathrm{head}}$~($\downarrow$)
& 22.109 $\pm$ 6.552
& 223.139 $\pm$ 84.019
& \textbf{3.822 $\pm$ 3.142} \\

& Success~($\uparrow$)
& \textbf{100.0\%}
& 10.0\%
& 0.0\% \\
\bottomrule
\end{tabular}
\end{table*}

\begin{figure*}[t]
    \centering
    \includegraphics[width=\textwidth]{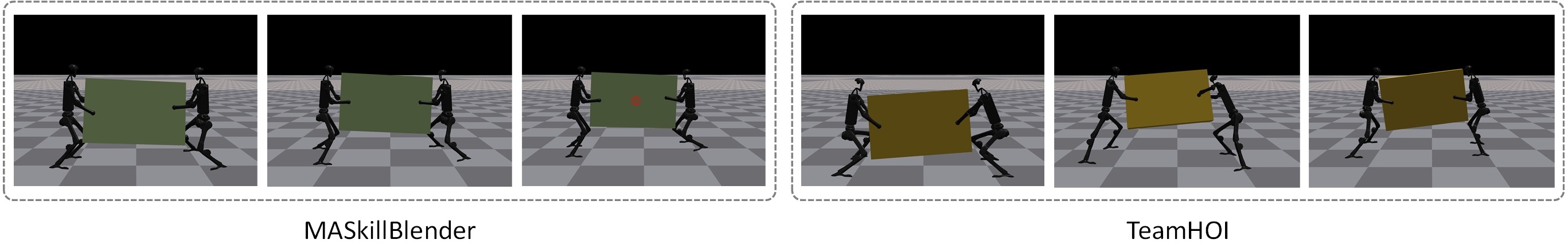}
    \caption{Qualitative comparison of MASkillBlender and TeamHOI on the Unitree H1 \texttt{Carry} task. Representative key frames from both methods are shown to illustrate the whole-body coordination behaviors during box lifting and transportation. }
    \label{fig:carry_qualitative_comparison}
\end{figure*}

\subsection{Main Results}  \label{subsec:main_results}
We evaluate all methods on the Unitree H1 using two independent trials per task, with the smoothed learning curves shown in Fig.~\ref{fig:training_curves_h1}. MASkillBlender achieves the highest final returns across all three tasks. TeamHOI shows large variation on \texttt{Carry} and \texttt{Push} and fails to learn \texttt{Move}. SkillBlender remains competitive on \texttt{Carry}, but shows reduced performance on \texttt{Push} and a substantial degradation on \texttt{Move}. This task-dependent gap suggests that directly extending single-humanoid skill blending to joint multi-humanoid control does not consistently provide effective coordination across tasks. For quantitative evaluation, we perform 50 rollouts with the best policy learned by each method, and the results are reported in Table~\ref{tab:main_results}. MASkillBlender achieves the highest success rate on all three tasks. The lower success rate on \texttt{Push} is mainly due to the heavy box and complex contact interactions among the humanoids, the box, and the ground. TeamHOI obtains a lower Error$_{\mathrm{head}}$ on \texttt{Move} because it learns to orient toward the target but fails to reach the target position. Fig.~\ref{fig:carry_qualitative_comparison} further provides a qualitative comparison on the \texttt{Carry} task. MASkillBlender achieves coordinated box transportation, whereas TeamHOI tends to move the box toward the target before establishing coordinated locomotion.

\subsection{Ablation Studies}  \label{subsec:ablation}
To assess the contribution of the key components in MASkillBlender, we conduct ablation studies on permutation-based data augmentation and primitive skill selection. We first remove permutation-based data augmentation while keeping all other settings unchanged. As shown in Fig.~\ref{fig:training_curves_h1}, removing the augmentation generally slows down learning and leads to lower final returns. The quantitative results in Table~\ref{tab:permutation_ablation} of Appendix~\ref{appendix_ablation_studies} show consistent performance drops across all three tasks. These results demonstrate the effectiveness of our data augmentation strategy. We further examine the contribution of the primitive skills by removing \texttt{Squatting} from \texttt{Carry} and \texttt{Walking} from \texttt{Push}. As shown in Fig.~\ref{fig:primitive_skill_ablation} of Appendix~\ref{appendix_ablation_studies}, both variants exhibit degraded learning performance, highlighting the importance of task-relevant primitive skills for coordinated whole-body control.

\subsection{Policy Robustness under Deployment Variations}  \label{subsec:policy_robustness}
To assess deployment robustness, we further evaluate MASkillBlender under several practical variations. Detailed experimental settings are provided in Appendix~\ref{appendix_policy_robustness}. First, we introduce localization errors, asynchronous control delays, and inter-robot communication delays into the task environments. Without additional training or fine-tuning, the learned policies maintain effective coordination across all three tasks. Then, we evaluate long-horizon execution from asymmetric initial configurations on \texttt{Carry} and \texttt{Push}. By decomposing the tasks and reusing pre-trained primitive skills, the learned coordination policies can be directly extended to these more complex initial configurations without additional policy training. Note that a brief discussion of the deployment advantages of our task-decomposition strategy is also provided in Appendix~\ref{appendix_policy_robustness}, where its reduced reliance on inter-robot communication helps mitigate the impact of communication delays.

\subsection{Sim2Sim Transfer}  \label{subsec:sim2sim}
To evaluate Sim2Sim transfer, we deploy MASkillBlender from Isaac Gym to MuJoCo~\citep{todorov2012mujoco} on \texttt{Carry} and \texttt{Move}. For \texttt{Carry}, the policy is briefly trained under broader domain randomization before deployment, while \texttt{Move} is transferred directly without additional training. MASkillBlender preserves strong task performance on both tasks after transfer to MuJoCo, demonstrating effective cross-simulator transfer. Detailed settings and results are provided in Appendix~\ref{appendix_sim2sim_transfer}.

\subsection{Additional Multi-Humanoid Coordination Experiments}  \label{subsec:additional_experiments}
To evaluate the versatility of MASkillBlender, we conduct additional experiments on different object-manipulation scenarios and humanoid embodiments. We introduce a two-humanoid box-exchange task on Unitree H1 to extend the evaluation beyond the manipulation of large shared objects. The results show that MASkillBlender can support coordinated manipulation involving smaller objects. We also evaluate MASkillBlender on the 21-DoF Unitree G1 across \texttt{Carry}, \texttt{Push}, and \texttt{Move}. The results suggest that MASkillBlender generalizes across different humanoid embodiments without substantial modifications. Detailed settings and results are provided in Appendices~\ref{appendix_box_exchange} and~\ref{appendix_g1_experiments}.

\section{Conclusion and Discussion}  \label{Conclusion and Discussion}
In this paper, we presented MASkillBlender, a general MARL framework for decentralized multi-humanoid coordination. By building a shared decentralized high-level policy over reusable pre-trained primitive skills, MASkillBlender focuses downstream learning on high-level skill composition and inter-agent coordination using local observations and task-level rewards. We further incorporated permutation-based data augmentation for homogeneous humanoid teams and analyzed its effect on policy optimization under the H-MG formulation. Extensive simulated experiments demonstrate the versatility of MASkillBlender across diverse multi-humanoid coordination tasks and humanoid embodiments, as well as its robustness under deployment variations and cross-simulator transfer.

Despite these encouraging results, several limitations remain. First, the current formulation assumes homogeneous humanoid teams and relies on a shared local policy. Extending the framework to heterogeneous teams with different morphologies, capabilities, or functional roles remains an important direction for future work. Second, the high-level policies are still trained and evaluated primarily in simulation with state-based observations. Although we evaluate deployment variations and Sim2Sim transfer, real-world multi-humanoid deployment still requires addressing onboard perception, state estimation, contact-rich sim-to-real transfer, and hardware safety. Third, the performance of MASkillBlender depends on the quality and coverage of the primitive skill library. Incorporating richer or automatically discovered skills may further broaden the range of multi-humanoid tasks that can be solved. Overall, we hope that this work can serve as a useful step toward general learning frameworks for coordinated humanoid robot teams.

\bibliography{reference}

\newpage
\appendix

\begin{center}
    {\Large\bfseries Appendix}
\end{center}

\section{Notations} \label{appendix_notation}
Denote $\mathcal{I} = \{1, \dots, N\}$. Denote $\mathbb{R}$ and $\mathbb{R}^m$ as the spaces of real numbers and $m$-dimensional real vectors, respectively. Let $x = (x_1, \dots, x_N)$ be an ordered list. For any ordered list with $N$ elements, we define a permutation $M = [m_1, \dots, m_N]$ such that $m_i \neq m_j$ for all $i \neq j$ and $m_k \in \mathcal{I}$ for all $k \in \mathcal{I}$. Let $\mathcal{M}$ denote the set of all such permutations. The permuted list of $x$ under $M$ is given by $Mx = (x_{m_1}, \dots, x_{m_N})$. For any vector $z = [z_1, \dots, z_m]^T \in \mathbb{R}^m$, we define the pointwise Softmax mapping $\mathrm{softmax}_c: \mathbb{R}^m \to \mathbb{R}$ by $\mathrm{softmax}_c(z) = \frac{e^{z_c}}{\sum_{i=1}^m e^{z_i}}$, where $c \in \{1, \dots, m\}$.

\section{An Example of H-MGs} \label{appendix_example}
As shown in Fig.~\ref{fig:versatile_task}, two humanoid robots collaboratively push a box to a target position. The target position is randomly initialized at the start of each episode and remains fixed for the duration of the episode. This task can be cast as an H-MG. Specifically, the local state $s_i$ consists of humanoid $i$'s proprioceptive state (e.g., its root and joint states), the root state of the box, and the task goal (i.e., the target box position). Since the humanoids are homogeneous, it holds that $\mathcal{S}_1 = \mathcal{S}_2$. Furthermore, the global state $s$ can be reconstructed by concatenating $s_1$ and $s_2$. The local action is the target joint position of each humanoid, and thus $\mathcal{A}_1 = \mathcal{A}_2$. This task contains a single team reward, which is composed of two sub-rewards. The first sub-reward depends on the distance between the box position and the target position. For the second sub-reward, we define two contact points on the box and assign each humanoid a unique contact point based on distance. The second sub-reward then depends on the distances between the humanoid's wrists and its assigned contact point. Since both sub-rewards are invariant to the ordering of the humanoids, the overall reward function is permutation invariant. The transition function is also permutation invariant because the underlying dynamics depend only on physics. Finally, the observation of each humanoid $i$ consists of its proprioceptive state, the relative position between its target and the box, and the relative positions between its wrists and its assigned contact point. Therefore, the joint observation is permutation-preserving.

\section{Primitive Skill Specifications}  \label{appendix_description_skills}

In our experiments, we consider three primitive skills, which are \texttt{Walking}, \texttt{Reaching}, and \texttt{Squatting}. Note that the primitive skills for the Unitree H1 are directly obtained from the project \citet{skillblender_github}, whereas those for the Unitree G1 are trained in-house.

\noindent$\bullet$\hspace{0.4em}\texttt{Walking}: This skill enables the humanoid to track commanded base motions for locomotion. Its goal is specified by the target planar linear velocity together with the target yaw angular velocity.

\noindent$\bullet$\hspace{0.4em}\texttt{Reaching}: This skill supports upper-body manipulation by driving the humanoid's two wrists toward designated 3D targets while keeping the body grounded. Its goal is defined as the relative offset from each wrist to its corresponding target.

\noindent$\bullet$\hspace{0.4em}\texttt{Squatting}: This skill modulates the humanoid's vertical posture by lowering or raising the body to a desired base height. Its goal is given by the relative height error between the current root and the target root height.

More detailed specifications of these primitive skills can be found in Appendix~C of~\citet{kuang2025skillblender}.

\section{Proofs} \label{appendix_proof}

\textbf{Proof of Lemma~\ref{lemma_value_function}.} Recall that the joint reward function $R(s,a) = (R_1(s,a), \dots, R_N(s,a))$ in H-MGs is permutation preserving as stated in Definition~\ref{definition:Homogeneous Markov game}. Since all agents share a team reward $R_{\mathrm{team}}$, it follows that $R_{\mathrm{team}}(Ms,Ma) = R_{\mathrm{team}}(s,a)$, which implies that $R_{\mathrm{team}}$ is permutation invariant. In addition, due to the permutation-preserving property of the joint observation, we have $o_j(Ms) = o_{m_j}(s)$ for all $j \in \mathcal{I}$ and any permutation $M = [m_1, \dots, m_N] \in \mathcal{M}$. Therefore, for any parameter $\theta$, it holds
\begin{align}  \label{lemma_formula_1}
    \pi^\mathrm{tot}_{\theta}(a|s) = \prod_{i=1}^N \pi_{\theta}(a_i|o_i(s)) = \prod_{j=1}^N \pi_{\theta}(a_{m_j}|o_{m_j}(s)) = \pi^\mathrm{tot}_{\theta}(Ma|Ms),
\end{align}
where the second equality follows from $\pi_{\theta}(a_i|o_i(s)) = \pi_{\theta}(a_{m_j}|o_{m_j}(s))$ for $i = m_j$ and all $i \in \mathcal{I}$. Equation~(\ref{lemma_formula_1}) indicates that the joint policy $\pi^\mathrm{tot}_{\theta}$ is permutation invariant. As a result, given the permutation invariance of the transition function, the reward function, and the joint policy, we have $V_{\theta}(s) = V_{\theta}(Ms)$ for all $M \in \mathcal{M}$. Furthermore, $Q_{\theta}(s,a) = Q_{\theta}(Ms,Ma)$ follows from the permutation invariance of the reward function and the value function.

\textbf{Proof of Theorem~\ref{theorem_data_augmentation}.} Recall that $\pi^\mathrm{tot}_{\theta}$ is permutation invariant for any parameter $\theta$ and permutation $M$, which implies
\begin{align}
    \frac{\pi^{\mathrm{tot}}_\theta(a_t|s_t)}{\pi^{\mathrm{tot}}_{\theta_{k}}(a_t|s_t)}
    &= \frac{\pi^{\mathrm{tot}}_\theta(M a_t|M s_t)}{\pi^{\mathrm{tot}}_{\theta_{k}}(M a_t|M s_t)}, \label{theorem_formula_1} \\
    \log \pi^{\mathrm{tot}}_{\theta}(a_t | s_t)
    &= \log \pi^{\mathrm{tot}}_{\theta}(M a_t | M s_t). \label{theorem_formula_2}
\end{align}
Denote $\theta = [\theta_1, \dots, \theta_p]^T \in \mathbb{R}^p$, and let $e_i$ be the $i$-th standard basis vector in $\mathbb{R}^p$. Then, for any $i \in \{1,\ldots,p\}$, it holds that
\begin{align}
      & \frac{\partial}{\partial \theta_i}\log \pi^{\mathrm{tot}}_{\theta}(a_t | s_t)  \notag \\
    = & \lim_{h \to 0} \frac{\log \pi^{\mathrm{tot}}_{\theta + h  e_i}(a_t | s_t) - \log \pi^{\mathrm{tot}}_{\theta}(a_t | s_t)}{h}  \notag \\
    = & \lim_{h \to 0} \frac{\log \pi^{\mathrm{tot}}_{\theta + h  e_i}(M a_t | M s_t) - \log \pi^{\mathrm{tot}}_{\theta}(M a_t | M s_t)}{h}  \notag \\
    = & \frac{\partial}{\partial \theta_i}\log \pi^{\mathrm{tot}}_{\theta}(M a_t | M s_t), \label{theorem_formula_3}
\end{align}
where $h \in \mathbb{R}$. Since (\ref{theorem_formula_3}) holds for all $i \in \{1,\ldots,p\}$, we obtain
\begin{align} \label{theorem_formula_4}
    \nabla_{\theta}\log \pi^{\mathrm{tot}}_{\theta}(a_t | s_t) = \nabla_{\theta}\log \pi^{\mathrm{tot}}_{\theta}(M a_t | M s_t).
\end{align}
In addition, since both $V_{\theta}(s)$ and $Q_{\theta}(s,a)$ are permutation invariant, it follows that
\begin{align} \label{theorem_formula_5}
    A_{\theta}(s_t, a_t) = A_{\theta}(M s_t, M a_t).
\end{align}
Combining (\ref{theorem_formula_1}), (\ref{theorem_formula_4}) and (\ref{theorem_formula_5}), we obtain (\ref{theorem_policy_gradient}), which completes the proof.

\section{Training Procedure of MASkillBlender} \label{appendix_maskillblender_training}
The detailed training procedure of MASkillBlender is summarized in Algorithm~\ref{alg:MASkillBlender}.

\begin{algorithm}[t]
\caption{MASkillBlender}
\label{alg:MASkillBlender}
\begin{algorithmic}[1]
\For{iteration $p=1,2,\ldots,P$}
    \For{environment $l=1,2,\ldots,L$ \textbf{in parallel}}
        \State All agents $i \in \mathcal{I}$ execute the decentralized policy $\pi_{\theta_{\mathrm{old}}}$ in environment $l$ for $T$ timesteps
        \State Compute advantage estimates $\hat{A}_1,\ldots,\hat{A}_T$ from collected trajectories
    \EndFor
    \For{epoch $k=1,2,\ldots,K$}
        \State Sample a minibatch $\mathcal{B}$ from the rollout buffer
        \State Apply permutation-based data augmentation to $\mathcal{B}$
        \State Update policy parameters $\theta$ and value parameters $\phi$ with MAPPO
    \EndFor
    \State $\theta_{\mathrm{old}} \gets \theta$
\EndFor
\end{algorithmic}
\end{algorithm}

\section{Experiment Details} \label{appendix_Experiment_Details}

\subsection{Observations and Actions}  \label{appendix_spaces}

The local observation of each humanoid $i$ consists of the proprioceptive state $s_{i,t}^{\mathrm{prop}}$ and the local goal $g_{i,t}$. Across all tasks, $s_{i,t}^{\mathrm{prop}}$ includes the joint angles, joint velocities, previous actions, base angular velocity, and base Euler angles. Let $d_{\mathrm{task}}$ denote the dimensionality of $g_{i,t}$, and recall that $d$ is the number of DoFs of the humanoid. The resulting local observation dimension for each humanoid is $3d + 6 + d_{\mathrm{task}}$.

Recall that the dimensionality of the skill-dependent goal vector for primitive skill $c$ is $d_c$. Therefore, the output dimensionality of the decentralized high-level policy is $\sum_{c=1}^C (d + d_c)$. Following existing works~\citep{kuang2025skillblender,he2024h2o,he2024omnih2o,zhang2024wococo}, both the hierarchical policy and the low-level skills output whole-body target joint positions for the humanoid, which are then converted into torques via a PD controller.

\subsection{Task Specifications}  \label{appendix_task_specifications}

\noindent$\bullet$\hspace{0.4em}\texttt{Carry}: The objective is to coordinate two humanoids to collaboratively carry a large box to a target position. The goal of each humanoid consists of the relative displacement between the box position and the target position, as well as the relative displacements between the humanoid's wrists and its assigned box-side reference point. The selected primitive skills are \texttt{Walking}, \texttt{Reaching}, and \texttt{Squatting}. The reward is defined as
\begin{equation}
    R_{\mathrm{team}}(s, a) = \omega_1 \cdot e^{-4 \Vert p_{\mathrm{box}} - p_{\mathrm{goal}} \Vert_1} + \omega_2 \cdot e^{-4 \|q_{\mathrm{wb}}\|_{\infty}},
\end{equation}
where $p_{\mathrm{box}} \in \mathbb{R}^3$ and $p_{\mathrm{goal}} \in \mathbb{R}^3$ are the current box position and the target box position, respectively, $q_{\mathrm{wb}} \in \mathbb{R}^{12}$ is the concatenated wrist-to-box-reference relative offsets of the two humanoids, and $\{\omega_i\}_{i=1}^2$ are weighting coefficients.

\noindent$\bullet$\hspace{0.4em}\texttt{Push}: The objective is to coordinate two humanoids to collaboratively push a large and heavy box to a target position on the plane. The goal of each humanoid consists of the relative displacement between the box position and the target position, as well as the relative displacements between the humanoid's wrists and its assigned box-side reference point. The selected primitive skills are \texttt{Walking} and \texttt{Reaching}. The reward is defined as
\begin{equation}
    R_{\mathrm{team}}(s, a) = \omega_1 \cdot e^{-4 \Vert p_{\mathrm{box}} - p_{\mathrm{goal}} \Vert_1} + \omega_2 \cdot e^{-4 \|q_{\mathrm{wb}}\|_{\infty}},
\end{equation}
where $p_{\mathrm{box}} \in \mathbb{R}^3$ and $p_{\mathrm{goal}} \in \mathbb{R}^3$ are the current box position and the target box position, respectively, $q_{\mathrm{wb}} \in \mathbb{R}^{12}$ is the concatenated relative offsets from the wrists of the two humanoids to their assigned box-side reference points, and $\{\omega_i\}_{i=1}^2$ are weighting coefficients.

\noindent$\bullet$\hspace{0.4em}\texttt{Move}: The objective is to coordinate three humanoids to move to their assigned target positions while avoiding collisions with the other humanoids. The goal of each humanoid consists of the relative displacement between its current root position and its target position, as well as the relative planar position of its closest neighboring humanoid. The selected primitive skill is \texttt{Walking}. The reward is defined as
\begin{equation}
    R_{\mathrm{team}}(s, a) = \omega_1 \cdot e^{-4 \bar{d}_{\mathrm{tar}}} + \omega_2 \cdot e^{-4 \bar{d}_{\mathrm{col}}} + \omega_3 \cdot e^{-4 \bar{e}_{\mathrm{head}}},
\end{equation}
where
\begin{equation}
    \bar{d}_{\mathrm{tar}} = \frac{1}{9} \sum_{i=1}^{3} \left\| p_i - p_i^{\star} \right\|_1,\quad
    \bar{d}_{\mathrm{col}} = \sum_{i=1}^{3} \max(0, 0.8 - \|c_i\|_2),\quad
    \bar{e}_{\mathrm{head}} = \frac{1}{3} \sum_{i=1}^{3} e_i^{\mathrm{head}}.
\end{equation}
Here, $p_i \in \mathbb{R}^3$ and $p_i^{\star} \in \mathbb{R}^3$ denote the current root position and the target position of humanoid $i$, respectively, $c_i \in \mathbb{R}^2$ denotes the relative planar position of the closest neighboring humanoid to humanoid $i$, $e_i^{\mathrm{head}}$ denotes the heading error of humanoid $i$ w.r.t. its target direction, and $\{\omega_i\}_{i=1}^3$ are weighting coefficients.

A comparison of three tasks is summarized in Table~\ref{tab:task_settings}. It is worth noting that each task relies on only a few reward terms, all of which are directly related to task completion. This design avoids extensive reward engineering.

\begin{table}[t]
\centering
\small
\setlength{\tabcolsep}{4pt}
\caption{Task settings.}
\label{tab:task_settings}
\begin{tabular}{C{1.5cm}C{1.7cm}C{1.9cm}C{1.4cm}C{1.5cm}C{1.6cm}}
\toprule
\multirow{2}{*}[-0.8ex]{Task}
& \multirow{2}{*}[-0.8ex]{\makecell[c]{Number of\\Humanoids}}
& \multirow{2}{*}[-0.8ex]{\makecell[c]{Number of\\Reward Terms}}
& \multicolumn{3}{c}{Selected Primitive Skills} \\
\cmidrule(lr){4-6}
& & & \texttt{Walking} & \texttt{Reaching} & \texttt{Squatting} \\
\midrule
\texttt{Carry} & 2 & 2 & $\checkmark$ & $\checkmark$ & $\checkmark$ \\
\texttt{Push}  & 2 & 2 & $\checkmark$ & $\checkmark$ &              \\
\texttt{Move}  & 3 & 3 & $\checkmark$ &              &              \\
\bottomrule
\end{tabular}
\end{table}

\subsection{Humanoid Model Details}  \label{appendix_model_details}
We consider two Unitree humanoid models in our experiments, namely the 19-DoF Unitree H1 and the 21-DoF Unitree G1. Note that the hand links are excluded for both models, and only the main upper-body and lower-body links relevant to locomotion and whole-body manipulation are retained. The corresponding embodiment details are summarized below:

\noindent$\bullet$\hspace{0.4em}\textbf{19-DoF Unitree H1}: This humanoid embodiment is the same as that used in~\citet{kuang2025skillblender}. It includes two 3-DoF shoulders, two 1-DoF elbows, one 1-DoF torso joint, two 3-DoF hips, two 1-DoF knees, and two 1-DoF ankle joints.

\noindent$\bullet$\hspace{0.4em}\textbf{21-DoF Unitree G1}: This humanoid embodiment is a modified 21-DoF Unitree G1 model derived from the official 23-DoF configuration file~\citep{unitree_ros}. It includes two 3-DoF shoulders, two 1-DoF elbows, one 1-DoF torso joint, two 3-DoF hips, two 1-DoF knees, and two 2-DoF ankle joints. In addition, we attach a virtual joint named \texttt{wrist\_target\_joint} to each arm to conveniently query the wrist position for task construction. These virtual joints are introduced only for kinematic querying and are not treated as actuated joints in control.

\begin{table}[t]
\centering
\small
\caption{Network hyperparameters of MASkillBlender for three Unitree H1 tasks.}
\label{tab:network_hyperparameters}
\begin{tabular}{C{2.3cm}C{3.9cm}C{3.9cm}}
\toprule
Task & Actor hidden dimensions & Critic hidden dimensions \\
\midrule
\texttt{Carry} & [512, 256, 128] & [768, 256, 128] \\
\texttt{Push} & [512, 256, 128] & [1024, 512, 128] \\
\texttt{Move} & [512, 256, 128] & [768, 256, 128] \\
\bottomrule
\end{tabular}
\end{table}

\subsection{Implementation Details}  \label{appendix_implementation_details}

\paragraph{MASkillBlender and SkillBlender.}
For all primitive skills, the policies are parameterized as end-to-end MLPs with three hidden layers of sizes [512, 256, 128]. Note that these primitive skill policies are pre-trained and kept fixed when training the high-level policy. For both methods, both the actor and critic are parameterized as MLPs with three hidden layers. Following~\citet{kuang2025skillblender}, the critic network takes privileged observations of all humanoids as input, including base linear velocity, body mass, and contact masks. For MASkillBlender, the network hyperparameters for all three tasks with the Unitree H1 are summarized in Table~\ref{tab:network_hyperparameters}. SkillBlender employs a larger actor network with hidden-layer sizes [768, 256, 128] for all tasks. In addition, a larger critic network with hidden-layer sizes [1024, 512, 128] is employed in MASkillBlender for all tasks with the Unitree G1. For the RL training hyperparameters, we set the entropy coefficient to $0.001$, the learning rate to $1 \times 10^{-5}$, the number of learning epochs to $2$, the discount factor to $\gamma = 0.994$, the GAE parameter to $\lambda = 0.9$, and the number of mini-batches to $4$. This hyperparameter setting is used for all tasks across all embodiments. For \texttt{Carry} and \texttt{Push}, we set 4096 parallel environments during training. For \texttt{Move}, we set 2048 parallel environments. In all tasks, MASkillBlender uses data generated from all permutations for policy training.

\paragraph{TeamHOI.}
We adapt TeamHOI to the Unitree H1 while preserving its core reference-based PPO and AMP training paradigm. Since the original TeamHOI is developed with a simplified 28-DoF humanoid, its motion references cannot be directly applied to the 19-DoF H1 used in our tasks. We therefore retarget the selected human motion references to the H1 embodiment using General Motion Retargeting (GMR)~\citep{araujo2026retargeting}. The retargeting process accounts for the humanoid's kinematic structure and joint limits, and converts the reference motions into H1-compatible root poses and joint trajectories for subsequent AMP training. For \texttt{Carry}, we select three pickup motions and two sideways locomotion motions from the motion data provided by TeamHOI and CooHOI, covering the main behaviors required for lifting and transporting the box. For \texttt{Push}, since no suitable pushing motion is available in the original TeamHOI motion set, we select two heavy-object pushing sequences from the CMU Motion Capture Database through AMASS~\citep{mahmood2019amass}. We further adjust their upper-body posture to obtain higher and more suitable hand-object interaction configurations before retargeting them to H1. For \texttt{Move}, we use locomotion references covering straight walking, stand-to-walk and walk-to-stand transitions, as well as mild left/right turns. Unlike the original TeamHOI, which switches between Full and Masked AMP according to the agent-object distance, our tasks initialize the humanoids directly near the interaction region. Therefore, we use either Full AMP or Masked AMP throughout training for each task. Based on empirical training performance, \texttt{Carry} and \texttt{Move} use Full AMP, while \texttt{Push} uses Masked AMP. We retain the shared decentralized Transformer policy of TeamHOI while adapting its observation interface to our benchmark tasks. For \texttt{Carry} and \texttt{Push}, each humanoid observes its proprioceptive state, task goal, and wrist-to-object interaction geometry, without explicit teammate information. For \texttt{Move}, the local observation additionally includes the relative position of the closest teammate for collision avoidance. Thus, the adapted TeamHOI policy uses the same task-level observation information as MASkillBlender for each task. Each observation component is encoded into a 64-D token and processed by three Transformer encoder layers with two attention heads and a feed-forward dimension of 512. The actor head outputs a 19-D action for the H1 humanoid, and the same architecture is used across all three tasks with task-dependent input dimensions. For task rewards, we replace the original TeamHOI transport-specific reward design with the reward functions defined in our benchmark environments, while retaining the AMP reward for motion regularization. The task reward and AMP reward are combined with equal weights during training. We retain the main PPO hyperparameters of the original implementation, including the learning rate, clipping ratio, value loss coefficient, discount factor, and GAE parameter, while adapting the rollout and optimization schedule to our training framework.

\paragraph{Training.}
We primarily employ NVIDIA RTX 5090 GPUs for policy training, with each training or inference run using a single GPU. For MASkillBlender on the Unitree H1, a single trial takes approximately 50 hours for \texttt{Carry}, 76 hours for \texttt{Push}, and 35 hours for \texttt{Move}. For the adapted TeamHOI baseline, the corresponding training times are approximately 60, 84, and 55 hours, respectively. Despite the larger number of entities involved in our task environments, all tasks remain trainable on a mid-range workstation equipped with a single NVIDIA RTX 4080 GPU.

\subsection{Ablation Studies}  \label{appendix_ablation_studies}
\paragraph{Permutation-Based Data Augmentation.} The quantitative results of the ablation on permutation-based data augmentation are provided in Table~\ref{tab:permutation_ablation}.

\paragraph{Primitive Skill Selection.} The learning curves of the primitive skill selection ablation are provided in Fig.~\ref{fig:primitive_skill_ablation}.

\begin{table*}[t]
\centering
\small

\caption{Ablation study of permutation-based data augmentation on the Unitree H1 tasks.}
\label{tab:permutation_ablation}

\begin{tabular}{C{1.4cm}C{2.1cm}C{3.1cm}C{3.5cm}}
\toprule
Task & Metric & MASkillBlender & MASkillBlender w/o Aug. \\
\midrule

\multirow{3}{*}{\texttt{Carry}}
& Return~($\uparrow$)
& \textbf{67.928 $\pm$ 1.563}
& 45.791 $\pm$ 5.684 \\

& Error$_{\mathrm{box}}$~($\downarrow$)
& \textbf{0.016 $\pm$ 0.008}
& 0.179 $\pm$ 0.068 \\

& Success~($\uparrow$)
& \textbf{100.0\%}
& 6.0\% \\

\cmidrule(lr){1-4}

\multirow{3}{*}{\texttt{Push}}
& Return~($\uparrow$)
& \textbf{61.137 $\pm$ 5.400}
& 57.130 $\pm$ 9.322 \\

& Error$_{\mathrm{box}}$~($\downarrow$)
& \textbf{0.116 $\pm$ 0.063}
& 0.147 $\pm$ 0.098 \\

& Success~($\uparrow$)
& \textbf{86.0\%}
& 76.0\% \\

\cmidrule(lr){1-4}

\multirow{4}{*}{\texttt{Move}}
& Return~($\uparrow$)
& \textbf{90.622 $\pm$ 1.022}
& 84.460 $\pm$ 1.147 \\

& Error$_{\mathrm{pos}}$~($\downarrow$)
& \textbf{0.010 $\pm$ 0.003}
& 0.023 $\pm$ 0.006 \\

& Error$_{\mathrm{head}}$~($\downarrow$)
& \textbf{22.109 $\pm$ 6.552}
& 73.530 $\pm$ 23.160 \\

& Success~($\uparrow$)
& \textbf{100.0\%}
& 86.0\% \\

\bottomrule
\end{tabular}

\end{table*}

\begin{figure*}[t]
    \centering
    \includegraphics[width=0.8\textwidth]{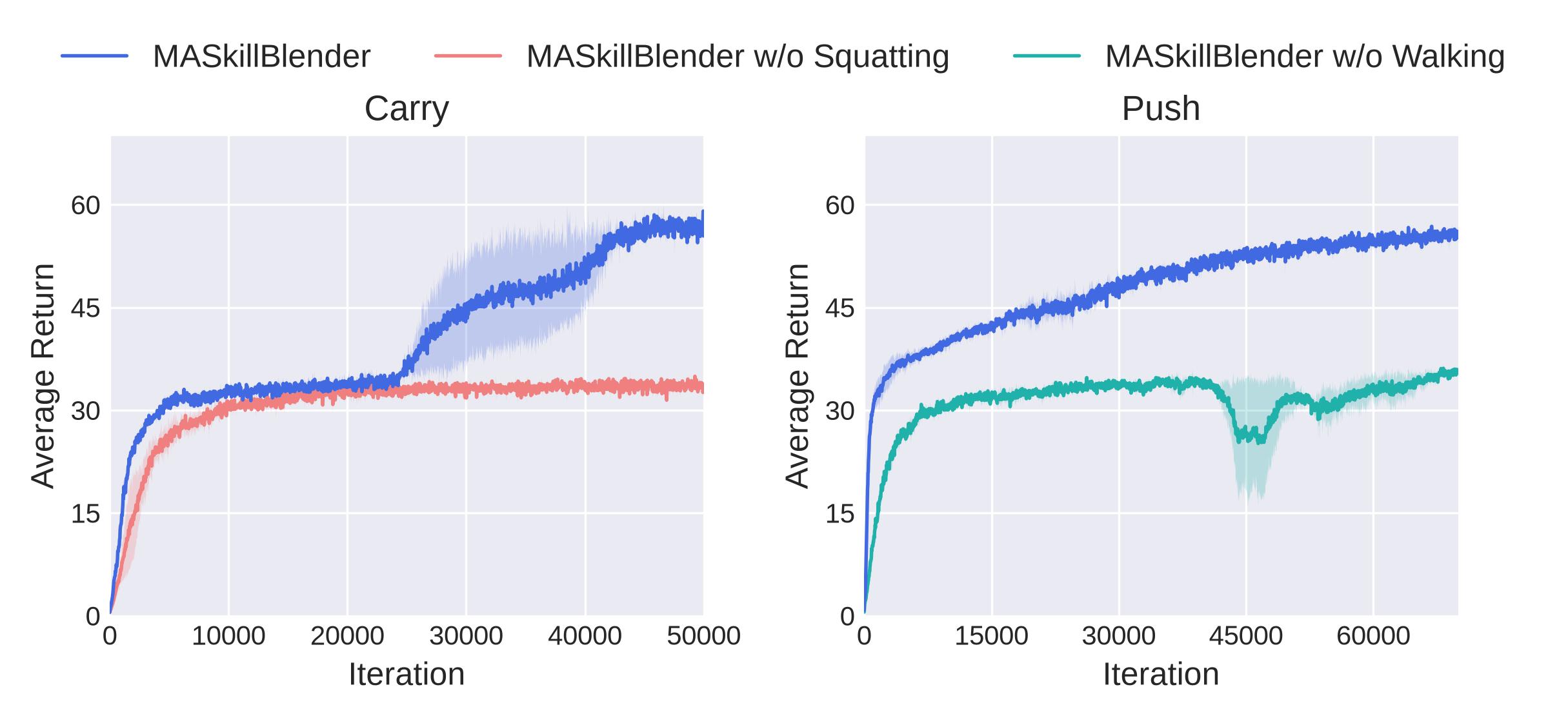}
    \caption{
    Primitive skill ablation on the Unitree H1. Removing \texttt{Squatting} from \texttt{Carry} or \texttt{Walking} from \texttt{Push} substantially degrades learning performance.
    }
    \label{fig:primitive_skill_ablation}
\end{figure*}

\begin{table*}[t]
\centering
\small

\caption{Robustness evaluation of MASkillBlender under deployment noise and delays on the Unitree H1 tasks.}
\label{tab:deployment_robustness}

\begin{tabular}{C{1.4cm}C{2.1cm}C{3.1cm}C{3.5cm}}
\toprule
Task & Metric & Nominal & Constrained \\
\midrule

\multirow{3}{*}{\texttt{Carry}}
& Return~($\uparrow$)
& \textbf{67.928 $\pm$ 1.563}
& 63.462 $\pm$ 11.406 \\

& Error$_{\mathrm{box}}$~($\downarrow$)
& \textbf{0.016 $\pm$ 0.008}
& 0.073 $\pm$ 0.201 \\

& Success~($\uparrow$)
& \textbf{100.0\%}
& 92.0\% \\

\cmidrule(lr){1-4}

\multirow{3}{*}{\texttt{Push}}
& Return~($\uparrow$)
& \textbf{61.137 $\pm$ 5.400}
& 52.744 $\pm$ 12.604 \\

& Error$_{\mathrm{box}}$~($\downarrow$)
& \textbf{0.116 $\pm$ 0.063}
& 0.158 $\pm$ 0.089 \\

& Success~($\uparrow$)
& \textbf{86.0\%}
& 74.0\% \\

\cmidrule(lr){1-4}

\multirow{4}{*}{\texttt{Move}}
& Return~($\uparrow$)
& \textbf{90.622 $\pm$ 1.022}
& 90.377 $\pm$ 2.289 \\

& Error$_{\mathrm{pos}}$~($\downarrow$)
& \textbf{0.010 $\pm$ 0.003}
& 0.014 $\pm$ 0.004 \\

& Error$_{\mathrm{head}}$~($\downarrow$)
& 22.109 $\pm$ 6.552
& \textbf{21.834 $\pm$ 10.831} \\

& Success~($\uparrow$)
& \textbf{100.0\%}
& 98.0\% \\

\bottomrule
\end{tabular}

\end{table*}

\subsection{Policy Robustness under Deployment Variations}  \label{appendix_policy_robustness}
To further evaluate MASkillBlender under more realistic deployment conditions, we consider two complementary settings beyond the nominal task environments. First, we construct constrained variants of \texttt{Carry}, \texttt{Push}, and \texttt{Move} by introducing localization errors, asynchronous control delays, and, for \texttt{Move}, inter-robot communication delays. Second, we evaluate long-horizon execution from asymmetric initial configurations on \texttt{Carry} and \texttt{Push}, where existing locomotion and whole-body skills are composed to approach the interaction region before the learned coordination policy takes over.

\paragraph{Robustness to Deployment Noise and Delays.}
For \texttt{Carry}, each humanoid independently experiences a control delay of 0, 10, or 20\,ms~\citep{zhang2024wococo}, which is sampled at the beginning of each episode and kept fixed throughout the rollout. To model localization errors in real-world perception and state estimation, we perturb the estimated box position by up to $\pm 0.02$\,m along each axis and its orientation by up to $\pm 0.1$\,rad. In addition, an independent positional error of up to $\pm 0.02$\,m is added to the positions of both wrists of each humanoid to account for inaccuracies in end-effector localization.

For \texttt{Push}, we apply the same independent control delays of 0, 10, or 20\,ms. Since the box remains on the ground throughout the task, we use relatively smaller perturbations for its estimated pose, with position errors of up to $\pm 0.02$\,m in the horizontal plane and $\pm 0.01$\,m vertically, and orientation errors of up to $\pm 0.01$\,rad for roll and pitch and $\pm 0.05$\,rad for yaw. Each humanoid also receives an independent wrist-position error of up to $\pm 0.02$\,m.

For \texttt{Move}, each humanoid independently experiences a control delay of 0, 10, or 20\,ms. To model localization uncertainty, its estimated base position is perturbed by up to $\pm 0.02$\,m in the horizontal plane and $\pm 0.01$\,m vertically. We further introduce inter-robot communication delays of 0, 10, 20, or 30\,ms for the exchanged position estimates. This models communication latency in decentralized multi-robot systems~\citep{chen2026rhythm}.

We directly evaluate the same policy checkpoint used in the nominal environment for 50 rollouts without any additional training or fine-tuning. As shown in Table~\ref{tab:deployment_robustness}, MASkillBlender maintains effective coordination across all three tasks under the introduced deployment noise and delays, which demonstrates its robustness to these deployment variations. This robustness may be partly attributed to the domain randomization already incorporated during training. We also observe a more noticeable drop in success rate on \texttt{Push}, which could potentially be mitigated by retraining the policy with broader domain randomization covering these deployment variations.

\begin{figure*}[t]
    \centering
    \includegraphics[width=\textwidth]{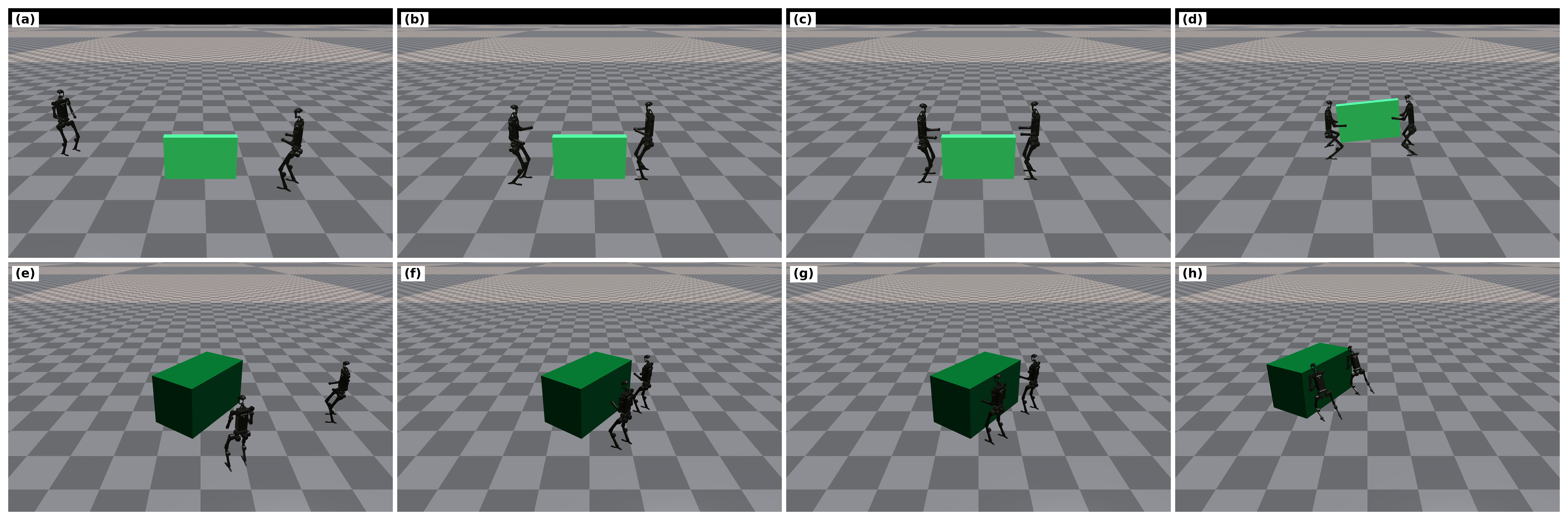}
    \caption{
    Long-horizon deployment from asymmetric initial configurations on the Unitree H1.
    The top row shows \texttt{Carry} and the bottom row shows \texttt{Push}.
    From left to right, the snapshots illustrate the asymmetric initial configuration, one humanoid arriving while the other is still approaching, both humanoids reaching the predefined task poses, and the final coordinated task execution.
    }
    \label{fig:long_horizon_asymmetric}
\end{figure*}

\paragraph{Long-Horizon Deployment from Asymmetric Initial Configurations.}
To facilitate efficient coordination learning, the humanoids in our benchmark tasks are initialized close to the object and interaction region. Nevertheless, following the idea of task decomposition~\citep{zhang2024wococo}, the learned coordination policies can be directly extended to asymmetric initial configurations without additional training by composing them with existing primitive skills.

For both \texttt{Carry} and \texttt{Push}, the long-horizon task is divided into three steps. The two humanoids first move from their asymmetric initial positions to predefined poses near the box. If one humanoid arrives earlier, it waits until the other is also in place. Once both humanoids are ready, they switch to the learned task policy to transport or push the box toward the target. The pre-trained \texttt{walking} skill is used to approach the box, while the \texttt{squatting} skill maintains a stable posture during waiting.

As shown in Fig.~\ref{fig:long_horizon_asymmetric}, this task decomposition allows the learned coordination policies to be directly extended to more complex asymmetric initial configurations. By reusing the pre-trained primitive skills for navigation and maintaining a stable waiting posture, MASkillBlender can accomplish long-horizon cooperative behaviors without additional policy training.

Existing multi-humanoid methods such as TeamHOI~\citep{lionar2026teamhoi} learn a unified policy to cover both approaching the object and the subsequent cooperative transportation. To enable behavioral transitions within a single policy, each humanoid explicitly observes the states of the other humanoids in addition to its own state and task-related information. Such peer information is useful during the approaching stage, since each humanoid needs to reason about whether its teammates have reached the interaction region and when coordinated manipulation should begin. However, once cooperative transportation starts, coordination can be achieved primarily from the humanoid's own state and the state of the shared object, without continuously observing explicit teammate states~\citep{gao2024coohoi}. Keeping peer-state observations throughout the entire unified policy therefore introduces an additional dependency during deployment. In real multi-robot systems, these states are typically obtained through inter-robot communication and relative localization, and can consequently be affected by communication delays and asynchronous state updates~\citep{chen2026rhythm}. In contrast, our task decomposition only requires lightweight communication to determine when both humanoids are ready to enter the cooperative stage. After the learned \texttt{Carry} or \texttt{Push} policy is activated, each humanoid executes the decentralized policy without explicit teammate-state communication. This design reduces the dependence of cooperative manipulation on continuous inter-robot communication and makes the execution less sensitive to communication latency.

\subsection{Sim2Sim Transfer}  \label{appendix_sim2sim_transfer}
To evaluate the cross-simulator transferability of MASkillBlender, we deploy the learned H1 policies from Isaac Gym to MuJoCo. We consider \texttt{Carry} and \texttt{Move}. For \texttt{Carry}, due to its contact-sensitive object manipulation, we initialize from the policy learned in Isaac Gym and continue training for a short period under a broader domain-randomization setting. For \texttt{Move}, we directly deploy the Isaac Gym policy in MuJoCo without additional training.

We evaluate both tasks over 100 episodes in MuJoCo, with the results reported in Table~\ref{tab:sim2sim_results}. For \texttt{Carry}, the two humanoids are able to carry the box through most episodes. In 92\% of the sampled episodes, the final box-position error is no greater than $0.1$\,m. For \texttt{Move}, MASkillBlender achieves a success rate of 99\%. Fig.~\ref{fig:sim2sim_mujoco} shows successful coordinated behaviors in MuJoCo. These results demonstrate that the learned policies can be effectively transferred across simulators.

\begin{figure*}[t]
    \centering
    \begin{minipage}[t]{0.40\textwidth}
        \centering
        \includegraphics[width=\linewidth]{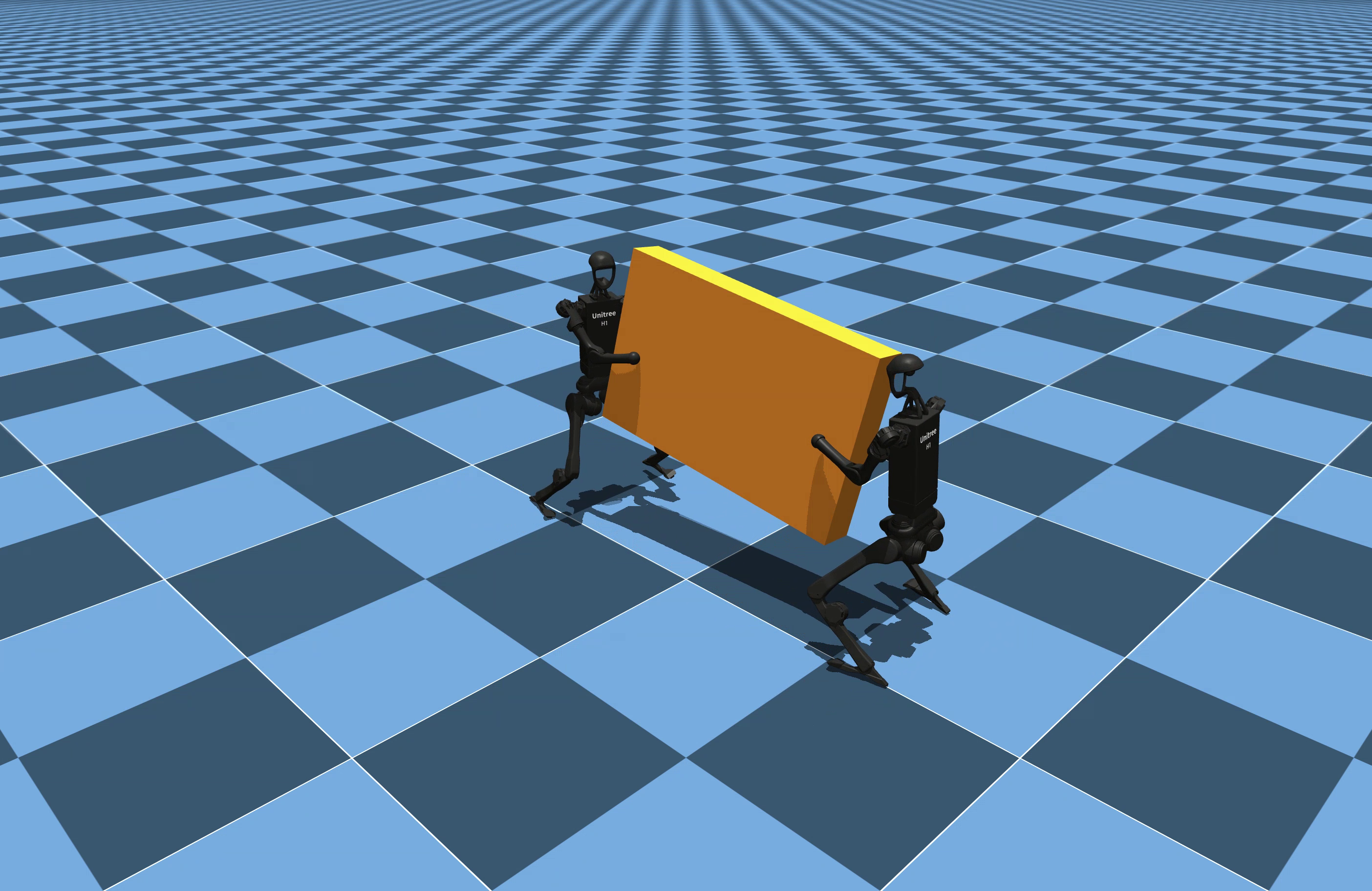}
        \vspace{-2mm}
        \centerline{(a) \texttt{Carry}}
    \end{minipage}
    \hspace{0.03\textwidth}
    \begin{minipage}[t]{0.40\textwidth}
        \centering
        \includegraphics[width=\linewidth]{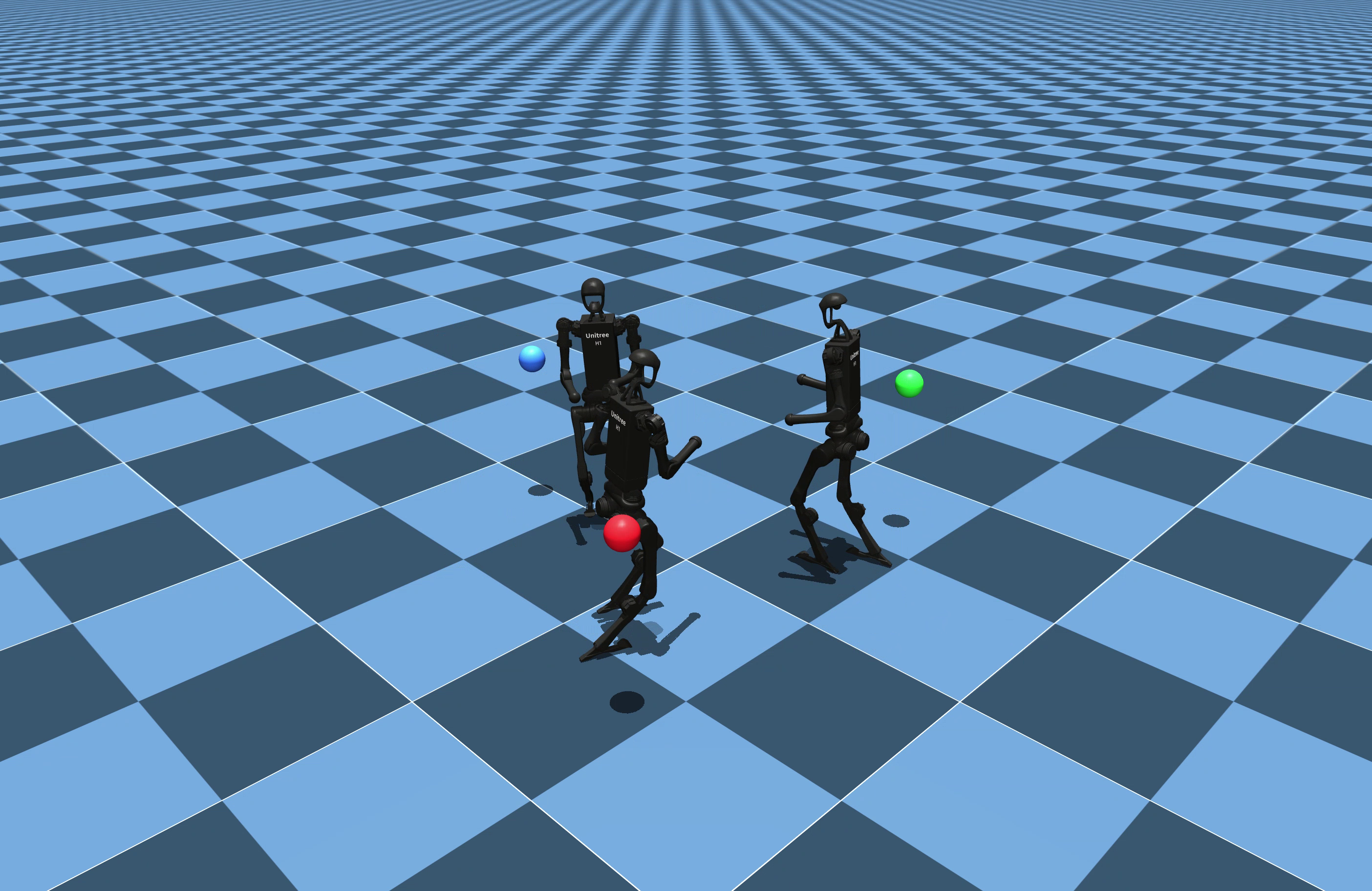}
        \vspace{-2mm}
        \centerline{(b) \texttt{Move}}
    \end{minipage}

    \caption{
    Sim2Sim deployment of MASkillBlender in MuJoCo on
    (a) \texttt{Carry} and (b) \texttt{Move}.
    }
    \label{fig:sim2sim_mujoco}
\end{figure*}

\begin{table}[t]
\centering
\small
\caption{Sim2Sim evaluation of MASkillBlender in MuJoCo on the Unitree H1 tasks.}
\label{tab:sim2sim_results}
\begin{tabular}{C{1.4cm}C{2.1cm}C{3.5cm}}
\toprule
Task & Metric & MuJoCo \\
\midrule
\multirow{2}{*}{\texttt{Carry}}
& Return~($\uparrow$)
& 66.787 $\pm$ 2.221 \\
& Error$_{\mathrm{box}}$~($\downarrow$)
& 0.026 $\pm$ 0.015 \\
\cmidrule(lr){1-3}
\multirow{3}{*}{\texttt{Move}}
& Return~($\uparrow$)
& 86.323 $\pm$ 1.430 \\
& Error$_{\mathrm{pos}}$~($\downarrow$)
& 0.013 $\pm$ 0.005 \\
& Error$_{\mathrm{head}}$~($\downarrow$)
& 38.151 $\pm$ 30.906 \\
\bottomrule
\end{tabular}
\end{table}

\subsection{Two-Humanoid Box Exchange}   \label{appendix_box_exchange}
We also evaluate MASkillBlender on the manipulation of smaller objects. Specifically, we design a two-humanoid box-exchange task, where two humanoids cooperatively exchange two small boxes across a table. As shown in Fig.~\ref{fig:box_exchange}(a), two humanoids stand on opposite sides of a table, with one box initially placed on each side. The task is divided into two stages. In the first stage, each humanoid moves its assigned box to a predefined exchange position near the center of the table. Once both boxes reach their exchange positions, the task switches to the second stage. The humanoids then exchange their assigned boxes and move them to the corresponding final target positions, which are the original positions of the opposite boxes.

For each humanoid, the task-related goal consists of the relative displacement between its currently assigned box and the corresponding target position, the relative displacements between its wrists and the assigned box, the current task stage, and the relative heights of the wrists and elbows with respect to the table. The selected primitive skills are \texttt{Walking} and \texttt{Reaching}. The task reward is composed of box-goal tracking, wrist-to-box proximity, and an arm-height term:
\begin{equation}
R_{\mathrm{team}}(s,a)
=
\omega_1 R_{\mathrm{box}}
+
\omega_2 R_{\mathrm{wrist}}
+
\omega_3 R_{\mathrm{height}},
\end{equation}
where $R_{\mathrm{box}}$ encourages the currently assigned boxes to approach their stage-dependent target positions, $R_{\mathrm{wrist}}$ encourages the humanoids' wrists to remain close to the corresponding boxes, and $R_{\mathrm{height}}$ encourages the wrists and elbows to remain above the table during manipulation. The box and wrist objectives are updated after the stage transition according to the exchanged box assignments. As shown in Fig.~\ref{fig:box_exchange}, the learned policy successfully coordinates the two humanoids throughout both stages of the task. The humanoids first move their assigned boxes to the exchange region and then switch responsibilities to deliver the exchanged boxes to their final target positions.

\begin{figure*}[t]
    \centering
    \includegraphics[width=\textwidth]{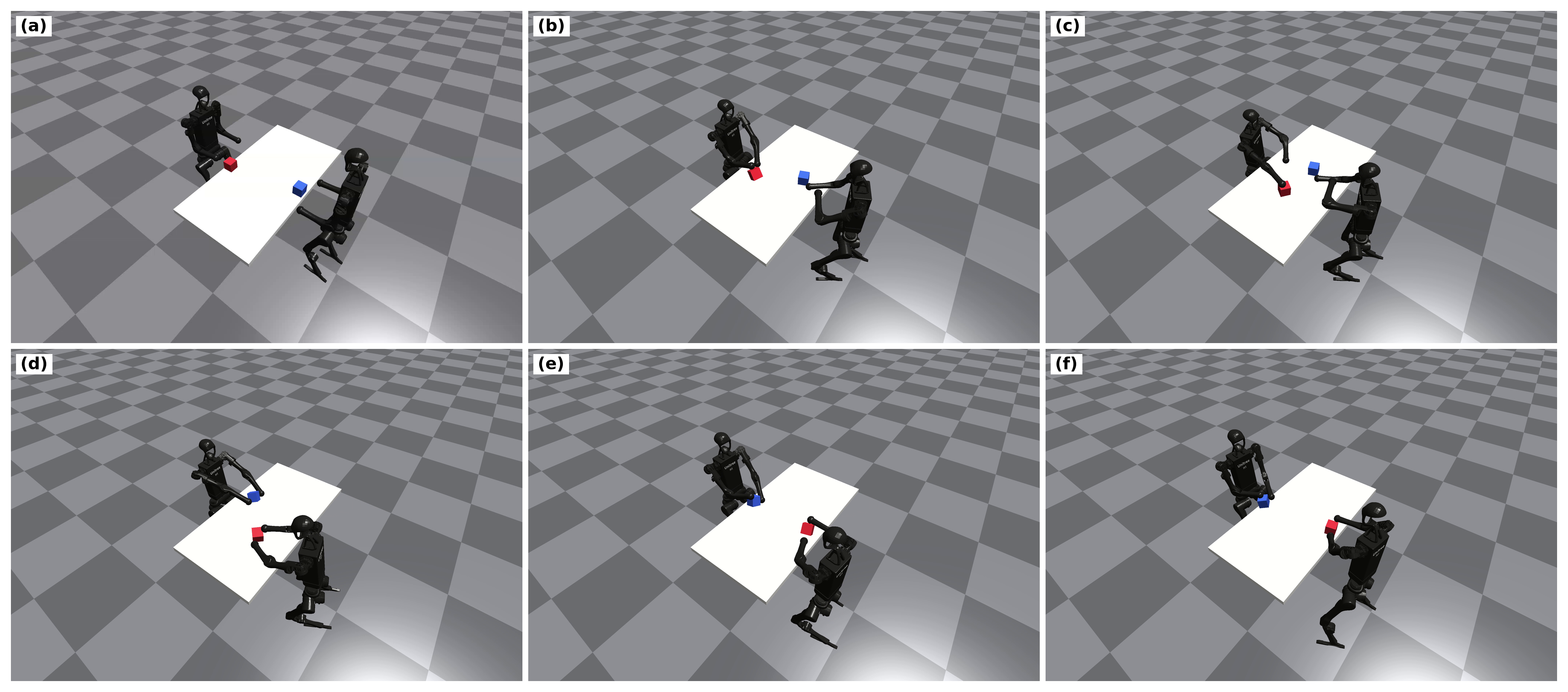}
    \caption{
    Two-humanoid box-exchange task on the Unitree H1.
    The snapshots illustrate the execution sequence from the initial configuration to the final box exchange.
    From (a) to (f), the two humanoids first move their assigned boxes to the exchange region and then deliver the exchanged boxes to their final target positions.
    }
    \label{fig:box_exchange}
\end{figure*}

\begin{figure*}[t]
    \centering
    \includegraphics[width=\linewidth]{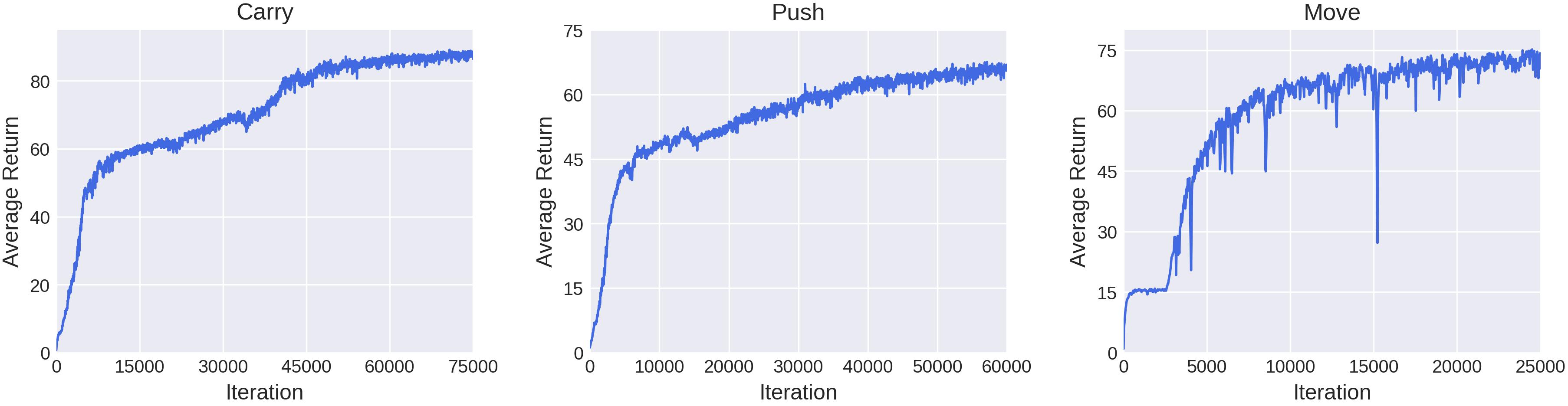}
    \caption{Training curves of MASkillBlender on three Unitree G1 tasks.}
    \label{fig:training_curves_g1}
\end{figure*}

\subsection{Additional Experiments on Unitree G1} \label{appendix_g1_experiments}
We finally evaluate MASkillBlender on the 21-DoF Unitree G1 to examine its applicability across different humanoid embodiments. As illustrated in Fig.~\ref{fig:versatile_task}, we consider the same three coordination tasks. The overall task definitions and learning framework remain unchanged. We only make minor empirical adjustments to the reward coefficients and training iterations for the G1 embodiment.

We provide the training curves on the three G1 tasks in Fig.~\ref{fig:training_curves_g1} and the quantitative results in Table~\ref{tab:g1_results}. The results show that MASkillBlender successfully learns coordinated behaviors across all three tasks. On \texttt{Carry} and \texttt{Push}, the box-position errors remain comparable to those obtained on the H1 embodiment. The G1 policy also learns coordinated movement on \texttt{Move}, although larger position and heading errors are observed. This difference may be related to the performance of the G1 walking primitive. Overall, these results suggest that MASkillBlender generalizes well across different humanoid embodiments and multi-humanoid coordination tasks, while maintaining strong performance without requiring substantial modifications to the overall framework.

\begin{table*}[t]
\centering
\small
\caption{Performance of MASkillBlender on the Unitree G1 tasks.}
\label{tab:g1_results}
\begin{tabular}{C{1.4cm}C{2.1cm}C{3.5cm}}
\toprule
Task & Metric & MASkillBlender \\
\midrule
\multirow{2}{*}{\texttt{Carry}}
& Return~($\uparrow$)
& 91.366 $\pm$ 5.339 \\
& Error$_{\mathrm{box}}$~($\downarrow$)
& 0.025 $\pm$ 0.043 \\
\cmidrule(lr){1-3}
\multirow{2}{*}{\texttt{Push}}
& Return~($\uparrow$)
& 68.851 $\pm$ 13.408 \\
& Error$_{\mathrm{box}}$~($\downarrow$)
& 0.113 $\pm$ 0.108 \\
\cmidrule(lr){1-3}
\multirow{3}{*}{\texttt{Move}}
& Return~($\uparrow$)
& 81.853 $\pm$ 2.115 \\
& Error$_{\mathrm{pos}}$~($\downarrow$)
& 0.039 $\pm$ 0.006 \\
& Error$_{\mathrm{head}}$~($\downarrow$)
& 54.420 $\pm$ 25.458 \\
\bottomrule
\end{tabular}
\end{table*}

\end{document}